%% file: main.tex
\documentclass{article}
\usepackage{graphics,epsfig}
\usepackage{amsthm}
\usepackage[latin1]{inputenc}
\usepackage[T1]{fontenc}
\usepackage{xcolor}
\usepackage{url}

\usepackage{times, tikz}
\usepackage{todonotes}

\usepackage{subfig}

\usepackage{amsfonts}
\usepackage{stmaryrd}
\usepackage{microtype}
\usepackage{booktabs}
\usetikzlibrary{shapes}

\usepackage{todonotes}
\usepackage{soul}

\usepackage{amscd, amsmath, amssymb, epic, eepic,  delarray}

\input{macros.tex}

\usepackage[normalem]{ulem}
\usepackage{mathtools}

\newcommand{\Set}[1]{\{{#1}\}}

\input{envt}

\def\cal#1{\ensuremath\mathcal #1}

\def\presuper#1#2%
  {\mathop{}%
   \mathopen{\vphantom{#2}}^{#1}%
   \kern-\scriptspace%
   #2}

\begin{document}

%============================= Title and Abstract
%\begin{frontmatter}
%\title{Explanation and Retraction: A general framework based on Satisfaction Systems, and Applications to Description Logics}
\title{Explanatory relations in arbitrary logics based on satisfaction systems, cutting and retraction}

\author{Marc Aiguier$^1$, Jamal Atif$^2$, Isabelle Bloch$^3$ and Ram\'on Pino-P\'erez$^4$ \\ \ \\
1. MICS, Centrale Supelec, Universit\'e Paris-Saclay, France, \\ {\it \{marc.aiguier\}@centralesupelec.fr} \\
2. Universit\'e Paris-Dauphine, PSL Research University, CNRS, UMR 7243, \\ LAMSADE, 75016 Paris, France, \\ {\it jamal.atif@dauphine.fr} \\
3. LTCI, T\'el\'ecom ParisTech, Universit\'e Paris-Saclay, Paris, France, \\ {\it isabelle.bloch@telecom-paristech.fr} \\
4. Departemento de Matem\'aticas, Facultad de Ciencias, \\ Universidad de Los Andes, M\'erida 5101, Venezuela \\ {\it pino@ula.ve}}

\date{}

\maketitle
\begin{abstract}
The aim of this paper is to introduce a new framework for defining abductive reasoning operators based on a notion of retraction in arbitrary logics defined as satisfaction systems.  We show how this framework leads to the design of  explanatory relations  satisfying properties of abductive reasoning, and discuss its application to several logics. This extends previous work on propositional logics where retraction was defined as a morphological erosion. Here weaker properties are required for retraction, leading to a larger set of suitable operators for abduction for different logics. 
\end{abstract}


%\end{frontmatter}

%================================ Start Text

%%

\input{introduction}

\input{institution}
\input{explanation}

\input{retraction}
\input{applications}
\input{conclusion}

\bibliographystyle{elsart-num-sort}
\bibliography{biblio}
\end{document}

%% file: macros.tex
\newcommand{\ALC}{\ensuremath{\mathcal{ALC}}}

\newcommand{\K}{\ensuremath{\mathcal{K}}}
\newcommand{\T}{\ensuremath{\mathcal{T}}}
\newcommand{\A}{\ensuremath{\mathcal{A}}}

\newcommand{\C}[1]{\ensuremath{\mathsf{C}({#1})}}

\newcommand{\nat}{\ensuremath{\mathbb{N}}}

\newcommand{\I}{\ensuremath{\mathcal{I}}}

\usepackage{ifthen}
\newcommand{\isdraft}{\boolean{true}} % if this is a draft version

\newcommand{\markupdraft}[2]{% {#1: {color|display} command}{#2: desired color or text}
    \ifthenelse{\equal{#1}{display}}{#2}{}%                 % display only in draft version
    \ifthenelse{\equal{#1}{color}}{\color{#2}}{}%           % colored only in draft (for \new command)
}
\newcommand{\newcolored}[3][]{{\markupdraft{color}{#2}#3}%  % kept in the final print
    \ifthenelse{\equal{#1}{}}{}{\markupdraft{display}{{\color{yellow!70!black}[#1]}}}} 

\ifthenelse{\isdraft}{}{\renewcommand{\markupdraft}[2]{}}

%% file: envt.tex
\newtheorem{definition}{Definition}
\newtheorem{notation}{Notation}
\newtheorem{lemma}{Lemma}
\newtheorem{proposition}{Proposition}
\newtheorem{theorem}{Theorem}
\newtheorem{corollary}{Corollary}
\newtheorem{example}{Example}

\newtheorem{remark}{Remark}
\newtheorem{cas}{Case}

\newenvironment{ex.}{
        \medskip
        \noindent {\bf Example}}
        {\hfill$\Box$ \medskip}

\newcommand{\LLE}{\mbox{\sf LLE}}

\def\RLE{\mbox{\sf RLE}}

\def\Econ{\mbox{\sf E-Con}}
\def\ECM{\mbox{\sf E-CM}}

\def\EDR{\mbox{\sf E-DR}}

\def\RA{\mbox{\sf RS}}
\def\ECC{\mbox{\sf E-C-Cut}}

\def\ERC{\mbox{\sf E-R-Cut}}

\def\ROR{\mbox{\sf ROR}}
\def\LOR{\mbox{\sf LOR}}
\def\ERef{\mbox{\sf E-Reflexivity}}

%% file: introduction.tex
\section{Introduction}

Since its introduction by Charles Peirce in~\cite{peirce1958collected}, abduction has motivated a large body of research in several scientific fields, e.g. philosophy of science, logics, law, artificial intelligence, to mention a few. Abduction, whatever the adopted view on its treatment, involves a background theory ($T$), an observation also called explanandum ($\varphi$), and an explanation ($\psi$). The observation may be seen as a surprising phenomenon that is  inconsistent with the background theory. It may also be consistent with the background theory but not directly entailed by this theory, which is the case considered in this paper. Several constraints can be imposed on the explanations and on the process of their production. One can allow changing the background theory, or not, consider as  non relevant explanations those that  entail the observation  on their own without engaging the background knowledge. Hence, several forms of abduction can be defined depending on the chosen criteria. Despite their divergence, most of these models agree to define abduction as an explanatory reasoning allowing us to infer the best explanation of an observation. This contributes to the field of explainable artificial intelligence. Explanatory relations, trying to model common sense and everyday reasoning, find applications in many domains, such as diagnosis~\cite{console1991,eiter1995}, forensics~\cite{han2011}, argumentation~\cite{booth2014a,booth2014b}, language understanding~\cite{hobbs2004}, image understanding~\cite{JA:SMC-14,YY:KI-15}, etc. (it is out of the scope of this paper to describe applications exhaustively). Then, as a form of inference, several works have studied rationality postulates that are more appropriate to govern the process of selecting the best explanations, e.g.~\cite{Flach96,PPU99}. From a computational point of view, a very large number of papers has tackled the definition of abductive procedures, mainly in propositional logics. An attractive approach, governed by what is called the AKM model, is based on semantic tableaux tailored for particular logics (e.g. propositional logics~\cite{Aliseda97}, first order and modal logics~\cite{MP93,MP95}), which was the basis for several extensions (e.g.~\cite{bienvenu2008,Britz17,eiter1995,halland2012}).  In these works, the explanatory reasoning process is split into two stages: (i) generating a set of hypotheses from the formulas that allow closing the open branches in the tableau constructed from $(T \cup \{\neg \varphi\})$, and (ii) selecting the preferred solutions from this plain set by considering some of the criteria discussed above. 

Our aim in this paper is to introduce a new framework for defining abductive reasoning operators in arbitrary logics in the framework of satisfaction systems. To this end, we propose on a new notion of cutting, from which operators of retraction are derived. We show that this framework leads to the design of explanatory relations satisfying the rationality postulates of abductive reasoning introduced in~\cite{PPU99} and adapted here to the proposed more general framework, and present applications in several logics. This extends previous work on abduction in propositional logics where retraction was defined as a morphological erosion~\cite{BL02,IB:arXiv-18,BPU04}, as well as abduction in description logics for image understanding~\cite{JA:SMC-14}. Here weaker properties are required for retraction, that allow defining a larger set of suitable operators for abduction for different logics. This approach is similar to the one proposed for revision in~\cite{AABH18}, where revision operators were defined from relaxation in satisfaction systems, and then instantiated in various logics.
An important feature of the proposed explanations based on retraction is that generation and selection steps are merged, or at least the set of generated hypotheses is reduced, thus facilitating the selection step.

The paper is organized as follows. In Section~\ref{institutions}, we recall the useful definitions and properties of satisfaction systems, and provide examples in propositional logic, Horn logic, first order logic, modal propositional logic and description logic. In Section~\ref{sec:expla} we introduce our first contribution, by defining a notion of cutting, from which explanations are then defined. In Section~\ref{sec:retraction}, we propose to define particular cuttings, based on retractions of formulas. Then in Section~\ref{applications}, we instantiate the proposed general framework in various logics.

%% file: institution.tex
\section{Satisfaction systems}
\label{institutions}

%\fbox{TO DO} Dire que c'est la meme chose que l'article sur la revision et enlever les preuves.

%The theory of institutions~\cite{GB92} is a categorical abstract model theory which emerges in computing science studies of software specification and semantics, in
%the context of the population explosion of logics there, with the ambition of doing as much
%as possible at the level of abstraction independently of commitment to any particular logic.
%Now institutions have become a common tool in the area of formal specification, in fact its
%most fundamental mathematical structure.

%Satisfaction systems~\cite{MDT09} generalize Tarski's classical ``semantic definition of truth''~\cite{Tar44} and Barwise's ``Translation Axiom''~\cite{Bar74}. For the sake of generalization, sentences are simply required to form a set. All other contingencies such as inductive definition of sentences are not considered. Similarly, models are simply seen as elements of a class, \emph{i.e.}~no particular structure is imposed on them.

We recall here the basic notions of satisfaction systems needed in this paper. The presentation follows the one in
%but omit the results of them. Interest readers may refer to our paper 
\cite{AABH18}, where we give a more complete presentation of satisfaction systems, including the properties and their proofs, that are omitted here.

\subsection{Definition and examples}
\label{sec:BasicDef}

\begin{definition}[Satisfaction system]
\label{satisfaction system} 

A {\bf satisfaction system} $\mathcal{R}=(Sen,Mod,\models)$ consists of
\begin{itemize}
\item a set $Sen$ of {\bf sentences},
\item a class $Mod$ of {\bf models}, and
\item a satisfaction relation $\models \subseteq Mod \times Sen$.
\end{itemize}
\end{definition}

 Let us note that the non-logical vocabulary, so-called {\em signature}, over which sentences and models are built, is not specified in Definition~\ref{satisfaction system}\footnote{The set of logical symbols is defined in each particular logic and does not depend on a theory.}. Actually, it is left implicit. Hence, as we will see in the examples developed in the paper, a satisfaction system always depends on a signature. 

\medskip
 \begin{example}
\label{examples of institutions}
The following examples of satisfaction systems are of particular importance in computer science and in the remainder of this paper.  
\begin{description}
\item[Propositional Logic (PL)] 
Given a set of propositional variables $\Sigma$, we can define the satisfaction system $\mathcal{R}_\Sigma = (Sen,Mod,\models)$ where $Sen$ is the least set
of sentences finitely built over propositional variables in $\Sigma$, the symbols $\top$ and $\bot$ (denoting tautologies and antilogies, respectively), and Boolean connectives in $\{\neg,\vee,\wedge,\Rightarrow\}$, $Mod$ contains all the  mappings $\nu:\Sigma\to\{0,1\}$ ($0$ and $1$ are the usual truth values), and the satisfaction relation $\models$ is the usual propositional satisfaction.
\item[Horn Logic (HCL)] A \emph{Horn clause} is a sentence of the form $\Gamma \Rightarrow
\alpha$ where $\Gamma$ is a finite (possibly empty) conjunction of propositional variables and $\alpha$
is a propositional variable. The satisfaction system of Horn clause logic is then defined as for {\bf PL} except that sentences are restricted to be conjunctions of Horn clauses.
\item[Modal Propositional Logic (MPL)] Given a set of propositional variables $\Sigma$, we can define the satisfaction system $\mathcal{R}_\Sigma = (Sen,Mod,\models)$ where 
\begin{itemize}
\item $Sen$ is the least set of sentences finitely built over propositional variables in $\Sigma$, the symbols $\top$ and $\bot$, Boolean connectives in $\{\neg,\vee,\wedge,\Rightarrow\}$, and modalities in $\{\Box,\Diamond\}$; 
\item $Mod$ contains all the Kripke models $(I,W,R)$ where $I$ is an index set, $W = (W^i)_{i \in I}$ is a family of functions from $\Sigma$ to $\{0,1\}$, and $R \subseteq I \times I$ is an accesibility relation;
\item the satisfaction of sentences by Kripke models, $(I,W,R) \models \varphi$, is defined by $(I,W,R) \models_i \varphi$ for every $i \in I$ where $\models_i$ is defined by structural induction on sentences as follows:
\begin{itemize}
\item $(I,W,R) \models_i p$ iff $p \in W^i$ for every $p \in \Sigma$,
\item Boolean connectives are handled as usual,
\item $(I,W,R) \models_i \Box \varphi$ iff $(I,W,R) \models_j \varphi$ for every $j \in I$ such that $(i,j) \in R$, and
\item $\Diamond \varphi$ is the same as $\neg \Box \neg \varphi$.
\end{itemize}
\end{itemize}
\item[First Order Logic (FOL) and Many-sorted First Order Logic] We detail here only the many-sorted variant of FOL, FOL being a particular case. Signatures are triplets $(S,F,P)$ where $S$ is a set
of sorts, and $F$ and $P$ are sets of function and predicate names respectively, both with arities in $S^\ast\times S$ and $S^+$ respectively ($S^+$ is the set of all non-empty sequences of elements in $S$ and $S^\ast=S^+\cup\{\epsilon\}$ where $\epsilon$ denotes the empty sequence). In the following, to indicate that a function name $f \in F$ (respectively a predicate name $p \in P$) has for arity $(s_1 \ldots s_n,s)$ (respectively $s_1 \ldots s_n$), we will note $f : s_1 \times \ldots \times s_n \to s$ (respectively $p : s_1 \times \ldots \times s_n$). \\ 
Given a signature $\Sigma=(S,F,P)$, we can define the satisfaction system $\mathcal{R}_\Sigma =(Sen,Mod,\models)$ where:
\begin{itemize}
\item $Sen$ is the least set of sentences built over atoms of the form $p(t_1,\ldots,t_n)$ where $p:s_1 \times \ldots \times s_n \in P$ and $t_i \in T_F(X)_{s_i}$ for every $i$, $1 \leq i \leq n$ ($T_F(X)_s$ is the term algebra of sort $s$ built over $F$ with sorted variables in a given set $X$) by finitely applying Boolean connectives in $\{\neg,\vee, \wedge, \Rightarrow \}$ and  quantifiers in $\{ \forall, \exists \}$. 
\item $Mod$ is the class of models $\mathcal{M}$ defined by a family $(M_s)_{s \in S}$ of non-empty sets (one for every $s \in S$), each one equipped with a function $f^{\cal M} : M_{s_1} \times \ldots \times M_{s_n} \rightarrow M_s$ for every $f:s_1 \times \ldots \times s_n \rightarrow s  \in F$ and with an n-ary relation $p^{\cal M} \subseteq M_{s_1} \times \ldots \times M_{s_n}$ for every $p:s_1 \times \ldots \times s_n \in P$. 
\item Finally, the satisfaction relation $\models$ is the usual first-order satisfaction. 
\end{itemize}
As for {\bf PL}, we can consider the logic {\bf FHCL} of first-order Horn Logic whose models are those of {\bf FOL} and sentences are restricted to be conjunctions of universally quantified Horn sentences (i.e. sentences of the form $\Gamma \Rightarrow \alpha$ where $\Gamma$ is a finite conjunction of atoms and $\alpha$ is an atom).
\item[Description logic (DL)] Signatures are triplets $(N_C,N_R,I)$ where $N_C$, $N_R$ and $I$ are nonempty pairwise disjoint sets where elements in $N_C$, $N_R$ and $I$ are called concept names, role names and individuals, respectively.  \\ Given a signature $\Sigma=(N_C,N_R,I)$, we can define the satisfaction system $\mathcal{R}_\Sigma =(Sen,Mod,\models)$ where:
\begin{itemize}
\item $Sen$ contains~\footnote{The description logic defined here is better known under the acronym $\mathcal{ALC}$.} all the sentences of the form $C \sqsubseteq D$, $x:C$ and $(x,y):r$ where $x,y \in I$, $r \in N_R$ and $C$ is a concept inductively defined from $N_C \cup \{ \top \}$ and binary and unary operators in $\{\_ \sqcap \_,\_ \sqcup \_\}$ and in $\{\neg\_,\forall r.\_,\exists r.\_\}$, respectively. 
\item $Mod$ is the class of models $\mathcal{I}$ defined by a set $\Delta^\mathcal{I}$ equipped for every concept name $A \in N_C$ with a set $A^\mathcal{I} \subseteq \Delta^\mathcal{I}$, for every relation name $r \in N_R$ with a binary relation $r^\mathcal{I} \subseteq \Delta^\mathcal{I} \times \Delta^\mathcal{I}$, and for every individual $x \in I$ with a value $x^\mathcal{I} \in \Delta^\mathcal{I}$. 
\item The satisfaction relation $\models$ is then defined as:
\begin{itemize}
\item $\mathcal{I} \models C \sqsubseteq D$ iff $C^\mathcal{I} \subseteq D^\mathcal{I}$,
\item ${\cal I} \models x:C$ iff $x^{\cal I} \in C^{\cal I}$,
\item ${\cal I} \models (x,y):r$ iff $(x^{\cal I},y^{\cal I}) \in r^{\cal I}$,
\end{itemize}
where $C^\mathcal{I}$ is the evaluation of $C$ in $\mathcal{I}$ inductively defined on the structure of $C$ as follows:
\begin{itemize}
\item if $C = A$ with $A \in N_C$, then $C^{\cal I} = A^{\cal I}$;
\item if $C = \top$ then $C^\mathcal{I} = \Delta^\mathcal{I}$;
\item if $C = C' \sqcup D'$ (resp. $C = C' \sqcap D'$), then $C^{\cal I} = C'^{\cal I} \cup D'^{\cal I}$ (resp. $C^{\cal I} = C'^{\cal I} \cap D'^{\cal I}$);
\item if $C = \neg C'$, then $C^{\cal I} = \Delta^{\cal I} \setminus C'^{\cal I}$;
\item if $C = \forall r.C'$, then $C^{\cal I} = \{x \in \Delta^{\cal I} \mid \forall  y \in \Delta^{\cal I}, (x,y) \in r^{\cal I} \mbox{ implies } y \in C'^{\cal I}\}$;
\item if $C = \exists r.C'$, then $C^{\cal I} = \{x \in \Delta^{\cal I} \mid \exists  y \in \Delta^{\cal I}, (x,y) \in r^{\cal I} \mbox{ and } y \in C'^{\cal I}\}$.
\end{itemize}
\end{itemize}
\end{description}
\end{example}

\subsection{Knowledge bases and theories}

Let us now consider a fixed but arbitrary satisfaction system  $\mathcal{R} = (Sen,Mod,\models)$ (since the signature $\Sigma$ is supposed fixed, the subscript $\Sigma$ will be omitted from now on). 

\medskip
\begin{notation}
Let $T \subseteq Sen$ be a set of sentences.
\begin{itemize}
\item $Mod(T)$ is the sub-class of $Mod$ whose elements are models of $T$, i.e. for every $\mathcal{M} \in Mod(T)$ and every $\varphi \in T$, $\mathcal{M} \models \varphi$. When $T$ is restricted to a formula $\varphi$ (i.e. $T = \{\varphi\}$), we will denote the class of model of $\{\varphi\}$ by $Mod(\varphi)$, rather than $Mod(\{\varphi\})$.
\item $Cn(T) =\{\varphi \in Sen \mid \forall {\cal M} \in Mod(T),~{\cal M} \models \varphi\}$ is the set of {\em semantic consequences of $T$}. In the following, we will also denote $T \models \varphi$ to mean that $\varphi \in Cn(T)$.
\item $\varphi \equiv_T \psi$ iff $Mod(T \cup \{\varphi\}) = Mod(T \cup \{\psi\})$.
%~\footnote{Usually, in the framework of satisfaction systems and institutions, the set of semantic consequences of a theory $T$ is noted $T^\bullet$. Here, we prefer the notation $Cn(T)$ because it will allow us to make a connection with the abstraction of logics as defined by Tarski~\cite{Tar56} and widely used in works dealing with belief change such as revision, expansion or contraction.}
\item Let $\mathbb{M} \subseteq Mod$. Let us note $\mathbb{M}^* = \{\varphi \in Sen \mid \forall \mathcal{M} \in \mathbb{M}, \mathcal{M} \models \varphi\}$. %Therefore, we have for every $T \subseteq Sen$, $Cn(T) = Mod(T)^*$. 
When $\mathbb{M}$ is restricted to one model $\mathcal{M}$, $\mathbb{M}^*$ will be equivalently noted $\mathcal{M}^*$.
\item Let us note $Triv = \{\mathcal{M} \in Mod \mid \mathcal{M}^* = Sen\}$, i.e. the set of models in which all formulas are satisfied. 
In {\bf PL}, {\bf MPL} and {\bf FOL}, $Triv$ is empty because the negation is considered. Similarly, the negation
%complementation 
is involved in the {\bf DL} $\mathcal{ALC}$, hence $Triv$ is empty. In {\bf HCL}, $Triv$ only contains the unique model where all propositional variables have a truth value equal to 1. In {\bf FHCL}, $Triv$ contains all models $\mathcal{M}$ where for every predicate name $p : s_1 \times \ldots \times s_n \in P$, $p^\mathcal{M} = M_{s_1} \times \ldots \times M_{s_n}$. 
% \Rev{As an example, $Triv$ is empty when the negation is considered, which is usual in propositional logic. If not, then the trivial model is the one where all propositional variables have a truth value equal to 1 in {\bf PL}.  }
\end{itemize}
\end{notation}
%Let us note that for every $T \subseteq Sen$, $Triv \subseteq Mod(T)$.

%From the above notations, we obviously have: 
%\begin{equation}
%Cn(T) = Cn(T') \Leftrightarrow Mod(T) = Mod(T').
%\label{eq:equivCnMod}
%\end{equation}

%The two functions $Mod(\_)$ from $\mathcal{P}(Sen)$ into $\mathcal{P}(Mod)$ and $\_^*$ from $\mathcal{P}(Mod)$ into $\mathcal{P}(Sen)$ form what is known as a Galois connection in that they satisfy the following properties: for all $T,T' \subseteq Sen$ and $\mathbb{M},\mathbb{M}' \subseteq Mod$, we have (see~\cite{Dia08} and the proof of Proposition~\ref{ref:propCn} below)
%\begin{enumerate}
%\item $T \subseteq T' \Longrightarrow Mod(T') \subseteq Mod(T)$
%\item $\mathbb{M} \subseteq \mathbb{M}' \Longrightarrow %{\mathbb{M}'}^* \subseteq \mathbb{M}^*$
%\item $T \subseteq Mod(T)^*$
%\item $\mathbb{M} \subseteq Mod(\mathbb{M}^*)$
%\end{enumerate}

\medskip
\begin{definition}[Knowledge base and theory]
\label{theory}
A {\bf knowledge base (KB)} $T$ is a finite set of sentences (i.e. $T \subseteq Sen$ and the cardinality of $T$ belongs to $\mathbb{N}$).
A set of sentences $T$ is said to be a {\bf theory} if and only if $T = Cn(T)$. \\ A theory $T$ is {\bf finitely representable} if there exists a KB $T' \subseteq Sen$  such that $T = Cn(T')$. \\ 
A class of models $\mathbb{M} \subseteq Mod$ is {\bf finitely axiomatizable} if there exists a finite KB $T$ such that $Mod(T) = \mathbb{M}$. A satisfaction system $\mathcal{R}$ is {\bf finitely axiomatizable} if each of its classes of models $\mathbb{M} \subseteq Mod$ is finitely axiomatizable.  
\end{definition}
Note that in DL, a knowledge base consists classically of a set of axioms (in the form $C \sqsubseteq D$), called TBox, and a set of assertions (in the form $x:C$ or $(x,y):r$), called ABox.

Classically, the consistency of a theory $T$ is defined as $Mod(T) \neq \emptyset$. The problem of such a definition of consistency is that its significance depends on the considered logic. Hence, this consistency is significant for {\bf FOL}, while in {\bf FHCL} it is a trivial property since each set of sentences is consistent because $Mod(T)$ always contains $Triv$ which is non empty. Here, for the notion of consistency to be more appropriate for our purpose of defining abduction for the largest family of logics, we propose a more general definition of consistency, the meaning of which is that given a theory $T$, $Mod(T)$ is not restricted to trivial models.

%there is at least a sentence which is not a semantic consequence. 

\medskip
\begin{definition}[Consistency]
$T \subseteq Sen$ is {\bf consistent} if $Cn(T) \neq Sen$.
\end{definition}

\medskip
\begin{proposition}[\cite{AABH18}]
\label{proposition:consistency}
For every $T \subseteq Sen$, $T$ is consistent if and only if $Mod(T) \setminus Triv \neq \emptyset$.
\end{proposition}
Hence, for every $T \subseteq Sen$, $T$ is inconsistent is equivalent to $Mod(T) = Triv$.
%\begin{proof} Let us prove that $Cn(T) = Sen$ iff $Mod(T) \setminus Triv = \emptyset$.
%Let us first assume that $Mod(T) \setminus Triv = \emptyset$. Therefore, this means that the only models that satisfy $T$ are $\mathcal{M}$ such that $\mathcal{M}^* = Sen$ (if they exist). Hence, we have $Cn(T) = Mod(T)^*=Sen$. \\
%Conversely, let us assume that $Cn(T) = Sen$. This means that every model $\mathcal{M}$ such that $\mathcal{M}^* \neq Sen$ does not belong to $Mod(T)$, and $Mod(T)\setminus Triv = \emptyset$. 
%\end{proof}

%\begin{corollary}
%\label{cor:triv}
%For every $T \subseteq Sen$, $T$ is inconsistent is equivalent to $Mod(T) = Triv$.
%\end{corollary}

\subsection{Internal logic}

Following~\cite{Dia08,GB92}, the satisfaction system-independent definition of Boolean connectives is straightforward. This will be useful when we give general results of preserving explanatory relation along Boolean connectives. Let $\mathcal{R}$ be a satisfaction system. A sentence $\varphi'$ is a

\begin{itemize}
\item {\bf semantic negation} of $\varphi$ when $Mod(\varphi') = Mod \setminus Mod(\varphi)$;
\item {\bf semantic conjunction} of $\varphi_1$ and $\varphi_2$ when $Mod(\varphi') = Mod(\varphi_1) \cap Mod(\varphi_2)$;
\item {\bf semantic disjunction} of $\varphi_1$ and $\varphi_2$ when $Mod(\varphi') = Mod(\varphi_1) \cup Mod(\varphi_2)$;
\item {\bf semantic implication} of $\varphi_1$ and $\varphi_2$ when $Mod(\varphi') = (Mod \setminus Mod(\varphi_1)) \cup Mod(\varphi_2)$.
\end{itemize}

$\mathcal{R}$ has (semantic) negation when each sentence has a negation. It has (semantic) conjunction (respectively disjunction and implication) when any two sentences have conjunction (respectively disjunction and implication). As usual, we note negation, conjunction, disjunction and implication by $\neg$, $\wedge$, $\vee$ and $\Rightarrow$.  

\medskip
\begin{example}
{\bf PL} has all semantic Boolean connectives. {\bf FOL} has all semantic Boolean connectives when sentences are restricted to closed formulas, otherwise (i.e. sentences can be open formulas) it only has semantic conjunction. Finally, {\bf MPL} has only semantic conjunction.
\end{example}

%% file: explanation.tex
\section{Explanation in satisfaction systems}
\label{sec:expla}

The process of inferring the best explanation of an observation is usually known as {\em abduction}. In a logic-based approach, the background of abduction is given by a knowledge base (KB) $T$ and a formula $\varphi$ (the observation) such that $T \cup \{\varphi\}$ is consistent. 
Besides this fact, which can be expressed equivalently as $T {\not \models} \neg \varphi$, some works further require that $T {\not \models} \varphi$.
We do not impose this last requirement here. 
%\Marc{Sometimes, some works further require that $T {\not \models} \varphi$ and $T {\not \models} \neg \varphi$.} \Isa{Si on avait $T \models \neg \varphi$, on ne pourrait pas avoir $T \cup \{ \varphi \}$ consistent ? Mentionner aussi qu'on n'impose pas $\psi \nvDash \varphi$ et commenter ?}

Let us start by introducing the notion of explanation of $\varphi$ with respect to $T$.

\medskip
\begin{definition}[Set of explanations]
\label{set of explanations}
Let $T$ be a KB. Let $\varphi \in Sen$ be a formula consistent with T. The {\bf set of explanations} of $\varphi$ over $T$ is the set $Expla_T(\varphi)$ defined as:
$$Expla_T(\varphi) = \{\psi \in Sen \mid Mod(T \cup \{\psi\}) \neq Triv~\mbox{and}~T \cup \{\psi\} \models \varphi\}$$
\end{definition}

Note that this definition does not impose that $\psi \not \models \varphi$. In some cases a preferred explanation of $\varphi$ with respect to the background knowledge base $T$ could be a formula $\psi$ such that $\psi \models \varphi$. 

Since abduction aims to infer the best explanations, the notion of explanation given in Definition~\ref{set of explanations} only captures candidate explanations of $\varphi$ with respect to $T$. Some additional properties are needed to define the key notion of ``preferred explanations". Following the works in~\cite{Aliseda97,Flach96,Flach00a,Flach00b,PPU99,PPU03}, we will study some preference criteria and give their logical properties when abduction is regarded as a form of inference.   

\medskip
\begin{definition}[Explanatory relation]
Let $T$ be a KB. An {\bf explanatory relation for $T$}  is a binary relation $\rhd \subseteq Sen \times Sen$ such that: 
$$\forall \varphi,\psi \in Sen, \varphi \rhd \psi \Longrightarrow \psi \in Expla_T(\varphi)$$
\end{definition}

Now, we define an (abstract) explanatory relation, the behavior of which will consist in cutting in $Mod(T \cup \{\varphi\})$ as much as possible but still under the constraint that it remains consistent (i.e. it is not equal to $Triv$). A cutting will then generate a sequence of subsets of $\mathcal{P}(Mod(T \cup \{\varphi\}))$ that we can order by inclusion. Moreover, this sequence cannot be extended by inverse inclusion. This gives rise to the notion of a cutting for a KB $T$ and a formula $\varphi$.

\medskip
\begin{definition}[Cutting]
\label{cutting}
Let $T$ be a KB and let $\varphi$ be a formula. A {\bf cutting} for $T$ and $\varphi$ is any $\mathcal{C} \subseteq \mathcal{P}(Mod(T \cup \{\varphi\})$ such that for every $\mathbb{M} \in \mathcal{C}$, $Triv \subsetneq \mathbb{M}$, $\mathcal{C}$ is closed under set-theoretical union and contains $Mod(T \cup \{\varphi\})$, and the poset $(\mathcal{C},\subseteq)$ is well-founded~\footnote{Let us recall that a poset $(X,\preceq)$ is well-founded if every non-empty subset $S \subseteq X$ has a minimal element with respect to $\preceq$, or equivalently there does not exist any infinite descending chain.}. \\ 
Let us denote $Min(\mathcal{C})$ the set of minimal elements for $\subseteq$ in $\mathcal{C}$. \\ 
In the following, given a KB $T$ and a formula $\varphi$, a cutting for $T$ and $\varphi$ will be denoted $\mathcal{C}_\varphi$.~\footnote{To simplify the notations, $T$ does not index cuttings because as we will see, $T$ will be often constant.}  
\end{definition}

Note that in Definition~\ref{cutting}, we do not impose that $T \cup \{\varphi\}$ is consistent (``who can do more, can do less''). However, the case where it is not would not be very interesting since $Mod(T \cup \{\varphi\}) \setminus Triv$ would then be empty.

\medskip
\begin{remark}\label{remark trivial}
If $Mod(T \cup \{\varphi\}) \neq Triv$ then there exists a trivial cutting for $\varphi$, namely 
$\mathcal{C}_\varphi=\Set{Mod(T \cup \{\varphi\})}$.
\end{remark}

As $(\mathcal{C}_\varphi,\subseteq)$ is closed under set-theoretical union and then it is inductive, by the Hausdorff maximal principle, every chain is contained in any maximal chain (and then maximal chains exist). Moreover, as $(\mathcal{C}_\varphi,\subseteq)$ is well-founded, every maximal chain has a least element which belongs to $Min(\mathcal{C}_\varphi)$. 

\medskip
\begin{definition}[Explanatory relation based on  cuttings]
\label{explanatory relation}
%Let $T$ be a KB and $\varphi$ a formula such that $Mod(T \cup \{\varphi\}) \neq Triv$. Let $\mathcal{C}_\varphi$ be a cutting for $T$ and $\varphi$. Let us define the binary relation $\rhd_{\mathcal{C}_\varphi} \subseteq Sen \times Sen$ as follows:
%$$\varphi  \rhd_{\mathcal{C}_\varphi} \psi \Longleftrightarrow 
Let $T$ be a KB, and let us define a set of cuttings $\mathcal{C}$ by choosing a cutting $\mathcal{C}_\varphi$ for every $\varphi$ in Sen: $\mathcal{C} = \{ \mathcal{C}_\varphi \mid \varphi \in Sen\}$.
%\ramon{ Suppose that for every formula $\varphi$  such that $Mod(T \cup \{\varphi\}) \neq Triv$, we have a cutting $\mathcal{C_\varphi}$  for $T$ and $\varphi$.}
 Let us define the binary relation $\rhd_\mathcal{C} \subseteq Sen \times Sen$ as follows:
$$\varphi  \rhd_{\mathcal{C}} \psi \Longleftrightarrow
\left\{\begin{array}{l}
Mod(T \cup \{\psi\}) \neq Triv,~\mbox{and} \\
\exists \mathbb{M} \in Min(\mathcal{C}_\varphi), Mod(T \cup \{\psi\}) \subseteq \mathbb{M}
\end{array}
\right.$$
\end{definition}
%\Isa{Ne faudrait-il pas changer la notation pour faire apparaitre la dependence en $\varphi$ (voire aussi en $T$) ? par exemple $\mathcal{C}_\varphi, \mathcal{C}_{\varphi,T}, \mathcal{C}(\varphi, T)$ ? }

By Remark~\ref{remark trivial}, $\mathcal{C}$ is well defined. Obviously, $\rhd_\mathcal{C}$ is an explanatory relation. We will later add some stability properties to $\mathcal{C}$ to ensure good properties of this explanatory relation.

\medskip
%The following observations are straightforward: \Isa{la 1e a \'et\'e mise plus haut et on s'en sert directement, mais je ne comprends pas celles-ci (et en a-t-on besoin ?)}
\begin{remark}\label{remark trivial2}
%\begin{enumerate}
%\item If $\mathcal{C}$    is a cutting for $T$ and $\varphi$ then there exist a family of cuttings $\mathcal{C}_{\varphi'}$ for each $\varphi'$ with  $Mod(T \cup \{\varphi\})\neq Triv$ and such that $\mathcal{C}_\varphi= \mathcal{C}$.
If $\mathcal{C}_\varphi$  is a cutting for $T$ and $\varphi$, then we can define a relation $\rhd_\mathcal{C}$ based on cuttings such that $\varphi\rhd_\mathcal{C} \psi$ satisfies the equivalence of Definition~\ref{explanatory relation} (i.e. $\mathcal{C}_\varphi$ is precisely the cutting chosen for $\varphi$ in the set $\mathcal{C}$).
%\end{enumerate}
\end{remark}
%}

The next example shows how our general definition via cuttings can capture some explanatory relations defined in the literature.

\medskip
\begin{example}
\label{abduction via semantic tableau}
Abduction via semantic tableau~\cite{MP93} and resolution~\cite{SNA06} generates a cutting, and then an explanatory relation. We illustrate this fact for abduction via semantic tableau in the framework of the propositional logic~\footnote{Note that semantic tableau methods have been extended to modal logic~\cite{Bienvenu09,MP95}, first-order logic~\cite{Marquis91,RAN06}, DL~\cite{halland2012}, etc., and in the same way we would be able to generate a cutting from them.}.

Semantic tableaux are used as refutation systems. Let $S$ be a set of propositional formulas
The tableau expansion rules are as follows: 
$$\begin{array}{llll}
\neg-rules: & \frac{S \cup \{\neg \neg \varphi\}}{S \cup \{\varphi\}} & \frac{S \cup \{\neg \bot\}}{S \cup \{\top\}} \\ \ \\
\alpha-rules: & \frac{S \cup \{\varphi_1 \wedge \varphi_2\}}{S \cup \{\varphi_1,\varphi_2\}} & \frac{S \cup \{\neg (\varphi_1 \Rightarrow \varphi_2)\}}{S \cup \{\varphi_1,\neg \varphi_2\}} & \frac{S \cup \{\neg(\varphi_1 \vee \varphi_2)\}}{S \cup \{\neg \varphi_1,\neg \varphi_2\}} \\ \ \\
\beta-rules: & \frac{S \cup \{\varphi_1 \vee \varphi_2\}}{\{S \cup \{\varphi_1\},S \cup \{\varphi_2\}\} }& \frac{S \cup \{\varphi_1 \Rightarrow \varphi_2\}}{\{S \cup \{\neg \varphi_1\},S \cup \{\varphi_2\}\}} & \frac{S \cup \{\neg(\varphi_1 \wedge \varphi_2)\}}{\{S \cup \{\neg \varphi_1\},S \cup \{\neg \varphi_2\}\}}
\end{array}$$

A tableau $\mathcal{T}$ is then a sequence of sets of sets of formulas $(\Gamma_1,\ldots,\Gamma_n,\ldots)$ such that, for every $i$, $\Gamma_{i+1}$ is obtained from $\Gamma_i$ by the application of a tableau expansion rule on a formula of a set $S$ in $\Gamma_i$. At each step $i$, every set $S$ in $\Gamma_i$ which contains both $p$ and $\neg p$ for some propositional variable $p$ is removed from $\Gamma_i$. \\ 
A formula $\varphi$ is a theorem of a KB $T$ if there exists a finite sequence $(\Gamma_1,\ldots,\Gamma_n)$ such that $\Gamma_1 = \{T \cup \{\neg \varphi\}\}$ and $\Gamma_n = \emptyset$. As an example, let us show that $a$ is a theorem of $\{a \wedge c,a \Rightarrow b\}$. The tableau method provides the finite sequence $\Gamma_1 = \{\{a \wedge c,a \Rightarrow b,\neg a\}\}, \Gamma_2 = \{\{a,c,a \Rightarrow b,\neg a\}\}$, using $\alpha$-rules. The set $\Gamma_2$ contains a unique set, with both $a$ and $\neg a$, which is then removed, and $\Gamma_2$ becomes empty.

Let us observe that the tableau expansion rules break propositional formulas on their main Boolean connectives. Hence, tableaux are necessarily finite, and then two cases can occur:
\begin{enumerate}
\item the last set $\Gamma_n$ of the sequence is empty, and then we have that $T \models \varphi$; or
\item every $S$ in $\Gamma_n$ only contains literals but no literal has its negation in $S$.
%none of them has its complementary in $S$.
\end{enumerate}
Following~\cite{MP93}, if $g$ is any consistent choice function for the elements of $\Gamma_n$, i.e. for $\Gamma_n = \{S_{n1},\ldots,S_{nm_n}\}$, $g(S_{ni}) \in S_{ni}$, then if $\psi = \neg g(S_{n1}) \wedge \ldots \wedge \neg g(S_{nm_n})$ is consistent with $T$, then $\psi$ is an explanation of $\varphi$ for $T$ ($\psi$ is even the minimal one according to the definition of minimality given in~\cite{MP93}). 

We now show that the way the tableau $\mathcal{T}$ is generated in~\cite{MP93} defines a cutting $\mathcal{C}_\mathcal{T}$. Before defining the cutting $\mathcal{C}_\mathcal{T}$, let us introduce some useful notions. Let $\mathcal{T} = (\Gamma_1,\ldots,\Gamma_n)$ be a tableau for $T \cup \{\neg \varphi\}$ such that $\Gamma_i = \{S_{i1},\ldots,S_{im_i}\}$. For every $j$, $1 \leq j \leq m_i$, let us denote $\psi_{ij}$ the disjunction of the negation of all the literals $l \in S_{ij}$, i.e. $\psi_{ij} = \bigvee\{\neg l \mid l : literal~\mbox{and}~l \in S_{ij}\}$. Then, let us set $\psi_i = \bigwedge_{1 \leq j \leq m_i} \psi_{ij}$. We can define the cutting $\mathcal{C}_\mathcal{T}$ as follows:
$$\mathcal{C}_\mathcal{T} = \{Mod(T \cup \{\varphi\})\} \cup (\cup_{1 \leq i \leq n}\{Mod(\psi_i)\})$$
Obviously we have $Mod(T \cup \{ \varphi\}) \in \mathcal{C}_\mathcal{T}$ and $Triv \notin \mathcal{C}_\mathcal{T}$. Moreover, for any $i$, $Mod(\psi_i) \subseteq Mod(T \cup \{\varphi\})$, hence $\mathcal{C}_\mathcal{T} \in \mathcal{P}(Mod(T \cup \{ \varphi \} ))$. It is not difficult to show that for every $i$, $1 \leq i \leq n$, $Mod(\psi_{i+1}) \subseteq Mod(\psi_i) \subseteq Mod(T \cup \{\varphi\})$. Moreover, the tableau $\mathcal{T}$ is finite, which completes the proof that $\mathcal{C}_\mathcal{T}$ is a cutting.

\medskip
Let us illustrate this construction on an example. Let $T = \{f \Rightarrow m,t \vee s,r \Rightarrow m\}$ be the KB and let $\varphi = m$ be the observation. The tableau method applied to $T \cup \{\neg \varphi\}$ generates four sets $\Gamma_1,\ldots,\Gamma_4$ where:
\begin{itemize}
\item $\Gamma_1 = \{\{f \Rightarrow m,t \vee s,r \Rightarrow m,\neg m\}\}$;
\item $\Gamma_2 = \{\{\neg f,t \vee s,r \Rightarrow m,\neg m\}\}$;
\item $\Gamma_3 = \{\{\neg f,t,r \Rightarrow m,\neg m\},\{\neg f,s,r \Rightarrow m,\neg m\}\}$;
\item $\Gamma_4  = \{\{\neg f,t,\neg r,\neg m\},\{\neg f,s,\neg r ,\neg m\}\}$;
\end{itemize}
This leads to the following formulas $\psi_1,\ldots,\psi_4$:
\begin{itemize}
\item $\psi_1 = m$;
\item $\psi_2 = f \vee m$;
\item $\psi_3 = (f \vee m \vee \neg t) \wedge (f \vee m \vee \neg s)$;
\item $\psi_4 = (f \vee m \vee \neg t \vee r) \wedge (f \vee m \vee \neg s \vee r)$.
\end{itemize}
A consistent choice satisfying minimality is for instance $f \vee r$.
\end{example}

%\Isa{On pourrait mettre un exemple d'Yifan pour illustrer ?}

\medskip
It is interesting to note that there is an alternative way of looking at $\rhd_{\mathcal{C}}$. The descending chains to obtain the minimal element $\mathbb{M}$ provide a method to order the models of $Mod(T \cup \{\varphi\})$. 

\medskip
\begin{definition}[Relation on models]
Let $T$ be a KB and let $\varphi$ be a formula such that $Mod(T \cup \{\varphi\}) \neq Triv$. Let $\mathcal{C}_\varphi$ be a cutting for $T$ and $\varphi$. Let us define $\preceq_{\mathcal{C}_\varphi} \subseteq Mod \times Mod$ as follows:
\begin{equation}
\mathcal{M} \preceq_{\mathcal{C}_\varphi} \mathcal{M}' \Longleftrightarrow 
\left\{
\begin{array}{l}
\exists C \subseteq \mathcal{C}_\varphi, s.t. C~\mbox{is a maximal chain} \\
\forall \mathbb{M} \in C, \mathcal{M}' \in \mathbb{M} \Rightarrow \mathcal{M} \in \mathbb{M}
\end{array}
\right. 
\label{eq:relationFromCutting}
\end{equation}
\end{definition}

Let $\mathbb{M} \subseteq Mod$ and $\preceq$ be a binary relation over $\mathbb{M}$. We define $\prec$ as $\mathcal{M} \prec \mathcal{M}'$ if and only if $\mathcal{M} \preceq \mathcal{M}'$ and $\mathcal{M}' \npreceq \mathcal{M}$. We also define $Min(\mathbb{M},\preceq) = \{\mathcal{M} \in \mathbb{M} \mid \forall \mathcal{M}' \in \mathbb{M}, \mathcal{M}' \nprec \mathcal{M}\}$. 
Note that the relation $\preceq_{\mathcal{C}_\varphi}$ is reflexive, but not necessarily transitive (hence it is not a pre-order).

\medskip
\begin{theorem}
Let $\mathcal{C}_\varphi$ be the cutting for a KB $T$ and a formula $\varphi$ used in the definition of $\rhd_\mathcal{C}$. For any $\psi \in Sen$, the following equivalence holds:\\ $\varphi \rhd_{\mathcal{C}}  \psi \Longleftrightarrow 
\left\{
\begin{array}{l}
Mod(T \cup \{\psi\}) \setminus Triv \neq \emptyset, \mbox{and} \\
Mod(T \cup \{\psi\}) \setminus Triv \subseteq Min(Mod(T \cup \{\varphi\}) \setminus Triv,\preceq_{\mathcal{C}_\varphi})
\end{array} 
\right.$
\end{theorem}

\begin{proof}
($\Rightarrow$) By definition of $\rhd_{\mathcal{C}}$, we have $Mod(T \cup \{\psi\}) \setminus Triv \neq \emptyset$. Let us suppose $\mathcal{M} \in  Mod(T \cup \{\psi\}) \setminus Triv$. By the definition of $\rhd_{\mathcal{C}}$, the statement $\varphi \rhd_{\mathcal{C}} \psi$ means that there exists $\mathbb{M} \in Min(\mathcal{C}_\varphi)$ such that $Mod(T \cup \{\psi\}) \subseteq \mathbb{M}$. As $(\mathcal{C}_\varphi,\subseteq)$ satisfies the Hausdorff maximal principle, there exists a maximal chain $C$, the least element of which is $\mathbb{M}$. Hence, by the definition of $\preceq_{\mathcal{C}_\varphi}$, for every $\mathbb{M}' \in C$, we have that:
\begin{enumerate}
\item for every $\mathcal{M}' \in \mathbb{M}$, $\mathcal{M} \preceq_{\mathcal{C}_\varphi} \mathcal{M}'$ and $\mathcal{M}' \preceq_{\mathcal{C}_\varphi} \mathcal{M}$, and
\item for every $\mathcal{M}' \in \mathbb{M}' \setminus \mathbb{M}$, $\mathcal{M} \prec_{\mathcal{C}_\varphi} \mathcal{M}'$.
\end{enumerate}
This proves that $\mathcal{M} \in Min(Mod(T \cup \{\varphi\}) \setminus Triv,\preceq_{\mathcal{C}_\varphi})$.

\medskip
($\Leftarrow$) Let us suppose that $\varphi \not\!\rhd_{\mathcal{C}} \psi$. This means that either $Mod(T \cup \{\psi\}) = Triv$ and in this case the conclusion is obvious, or there does not exist a minimal element $\mathbb{M} \in  Min(\mathcal{C}_\varphi)$ such that $Mod(T \cup \{\psi\}) \subseteq \mathbb{M}$. Let $\mathbb{M}$ be the least element (for inclusion) of $\mathcal{C}_\varphi$ such that $Mod(T \cup \{\psi\}) \subseteq \mathbb{M}$. This least element  $\mathbb{M}$ exists because $\mathcal{C}_\varphi$ contains $Mod(T \cup \{\varphi\})$ and $(\mathcal{C}_\varphi, \subseteq)$ is well-founded. As $(\mathcal{C}_\varphi,\subseteq)$ satisfies the Hausdorff maximal principle, there exists a maximal chain $C$ which contains $\mathbb{M}$, and then $\mathbb{M}$ cannot be the least element of $C$. Therefore, there exist some models $\mathcal{M}'$ which belong to some elements $\mathbb{M}'$ in $C$ such that $Mod(T \cup \{\psi\}) \nsubseteq \mathbb{M}'$ whence we can deduce that for some models $\mathcal{M} \in Mod(T \cup \{\psi\}) \setminus Triv$, we have $\mathcal{M}' \prec_{\mathcal{C}_\varphi} \mathcal{M}$.
\end{proof}

The explanatory relation $\rhd_{\mathcal{C}}$ satisfies a number of logical properties. Most of these properties are (rationality) postulates defined in~\cite{PPU99} up to some adaptations.
Let us recall them, adapted to the satisfaction system context, for any KB $T$, explanatory relation $\rhd$ for $T$ and formulas $\varphi,\varphi',\psi \in Sen$:

$$\begin{array}{ll}
\LLE & \frac{\varphi \equiv_T \varphi' ~~~ \varphi \rhd \psi}{\varphi' \rhd \psi} \\ 
\RLE\ & \frac{\psi\equiv_T \psi' ~~~  \varphi \rhd \psi}{\varphi \rhd \psi'} \\ 
\ECM & \frac{\varphi \rhd \psi ~~~  T \cup \{\psi\} \models \varphi'}{\varphi \wedge \varphi' \rhd \psi} \\
\ECC & \frac{\varphi \wedge \varphi' \rhd \psi ~~~  \forall \psi' (\varphi \rhd \psi' \Rightarrow T \cup \{\psi'\} \models \varphi')}{\varphi \rhd \psi} \\ 
\ERC & \frac{\varphi \wedge \varphi' \rhd \psi ~~~  \exists \psi' (\varphi \rhd \psi' ~\mbox{and}~ T \cup \{\psi'\} \models \varphi')}{\varphi \rhd \psi} \\ 
\LOR & \frac{\varphi \rhd \psi ~~~  \varphi' \rhd \psi}{\varphi \vee \varphi' \rhd \psi} \\ 
\EDR & \frac{\varphi \rhd \psi ~~~  \varphi' \rhd \psi'}{\varphi \vee \varphi' \rhd \psi~\mbox{or}~\varphi \vee \varphi' \rhd \psi'} \\ 
\ROR & \frac{\varphi \rhd \psi ~~~  \varphi \rhd \psi'}{\varphi \rhd \psi \vee \psi'} \\
\RA & \frac{\varphi \rhd \psi ~~~  \K \cup \{\psi'\} \models \psi ~~~ Mod(T \cup \{\psi'\}) \neq Triv}{\varphi \rhd \psi'} \\ 
\ERef & \frac{\varphi \rhd \psi}{\psi \rhd \psi} \\ 
\Econ & Mod(T \cup \{\varphi\}) \neq Triv \Longleftrightarrow \exists \psi, \varphi \rhd \psi
\end{array}$$

Now, we will show that, with an appropriate structure on the set of cuttings $\mathcal{C}$, adding a limited set of rather intuitive stability and monotony requirements, we can get strong results on the explanatory relation $\rhd_\mathcal{C}$, according to the above postulates. Recall that $\mathcal{C}$ is defined by choosing a cutting $\mathcal{C}_\varphi$ for each $\varphi$ in $Sen$. A first requirement is that for every $\varphi, \varphi'$ we have:
\begin{equation}
\mbox{If } \varphi \equiv_T \varphi', \mbox{ then } \mathcal{C}_\varphi=\mathcal{C}_{\varphi'}
\label{eq:equivCutting}
\end{equation}
This will be directly used in Property (1) of the following Theorem.

\medskip
\begin{theorem}
\label{logical properties 1}
Let $\mathcal{R}$ be a satisfaction system, $T$ a KB, $\mathcal{C}$ a set of cuttings and $\rhd_\mathcal{C}$ the explanatory relation based on cuttings of
Definition~\ref{explanatory relation}.
%$\mathcal{C}_\varphi$ a cutting for a KB $T$ and a formula $\varphi$.
The following properties are satisfied, for every  $\varphi, \varphi', \psi, \psi'$:
\begin{enumerate}
\item Assume that $\mathcal{C}$ satisfies Equation~\ref{eq:equivCutting}. If $\varphi \equiv_T \varphi'$, $\psi \equiv_T \psi'$ and $\varphi \rhd_{\mathcal{C}} \psi$, then $\varphi' \rhd_{\mathcal{C}} \psi'$.
\item If $\varphi \rhd_{\mathcal{C}} \psi$ and $T \cup \{\psi'\} \models \psi$ with $Mod(T \cup \{\psi'\}) \neq Triv$, then $\varphi \rhd_{\mathcal{C}} \psi'$.
\item $\psi\in Expla_T(\varphi)$ iff there exists a relation $\rhd_{\mathcal{C}}$ based on cuttings such that $\varphi \rhd_{\mathcal{C}} \psi$.
%$Expla_T(\varphi) \neq \emptyset \Longleftrightarrow \exists \mathcal{C}_\varphi:\mbox{cutting for } T \mbox{ and } \varphi, \exists \psi \in Sen, \varphi \rhd_{\mathcal{C}_\varphi} \psi$.
\item If $\mathcal{R}$ is finitely axiomatizable for every $\mathbb{M} \subseteq Mod$ and has conjunction, then for every cutting $\mathcal{C}_\varphi$, we have that $Expla_T(\varphi) \neq \emptyset$ and $\exists \psi \in Sen, \varphi \rhd_{\mathcal{C}} \psi$, 
where  $\rhd_{\mathcal{C}}$ is a relation based on cuttings such that the cutting associated with $\varphi$ is $\mathcal{C}_\varphi$.
\end{enumerate}
\end{theorem}

\begin{proof}
\begin{enumerate}
\item The first property is obviously satisfied because  $\varphi \equiv_T \varphi'$ and $\psi \equiv_T \psi'$ mean that $Mod(T \cup \{\varphi\}) = Mod(T \cup \{\varphi'\})$ and $Mod(T \cup \{\psi\}) = Mod(T \cup \{\psi'\})$ and, by assumption, $\mathcal{C}_\varphi=\mathcal{C}_{\varphi'}$.
\item $\varphi \rhd_{\mathcal{C}} \psi$ means that there exists $\mathbb{M} \in Min(\mathcal{C}_\varphi)$ such that $Mod(T \cup \{\psi\}) \subseteq \mathbb{M}$, and $Mod(T \cup \{ \varphi \}) \neq Triv$. As $Mod(T \cup \{\psi'\}) \subseteq Mod(T \cup \{\psi\})$ (hence $Mod(T \cup \{ \psi'\}) \subseteq \mathbb{M}$), and $Mod(T \cup \{ \psi'\}) \neq Triv$, we can deduce that $\varphi \rhd_{\mathcal{C}} \psi'$.
\item The ``if part'' is obvious. To prove the ``only if part'', let us notice that for every $\psi \in Expla_T(\varphi)$, we can build in $\mathcal{P}(Mod(T \cup \{\varphi\})$ a saturated chain $\mathcal{C}_\varphi$ starting at $Mod(T \cup \{\psi\})$. As $\psi \in Expla_T(\varphi)$, this saturated chain satisfies all the conditions of Definition~\ref{cutting}, and then it is a cutting for $T$ and $\varphi$. By Remark~\ref{remark trivial2}, we can define the relation $\rhd_\mathcal{C}$ based on cuttings such that the cutting corresponding to $\varphi$ is precisely $\mathcal{C}_\varphi$. By construction of $\mathcal{C}_\varphi$, it is clear that $\varphi\rhd_\mathcal{C} \psi$.
\item Again by Remark~\ref{remark trivial2}, 
%there exists a family of cuttings for $T$ such that the cutting corresponding to $\varphi$ is precisely $\mathcal{C}_\varphi$. 
it makes sense to consider $\rhd_\mathcal{C}$ as the explanatory relation defined by a family of cuttings in which the one associated with $\varphi$ is $\mathcal{C}_\varphi$.
Let $\mathbb{M} \in Min(\mathcal{C}_\varphi)$. As $\mathcal{R}$ is finitely axiomatizable, there exists a finite KB $T'$ such that $Mod(T') = \mathbb{M}$. Let us set $\psi = \bigwedge_{\varphi' \in T'} \varphi'$. We obviously have that $T \cup T'$ is consistent, hence $Mod(T \cup \{ \psi \}) \neq Triv$, and then we can conclude that $\varphi \rhd_{\mathcal{C}} \psi$.
\end{enumerate}
\end{proof}

It is interesting to note that property (1) generalizes to satisfaction systems the properties \LLE\ and \RLE\ of~\cite{PPU99}. Similarly, Property (2) corresponds to \RA, and Properties (3) and (4) to \Econ.
%\Marc{In~\cite{PPU99}, the properties $(1)$ and $(2)$ are the rationality postulates known under the acronyms LLE and RLE, and RA, respectively.}

If $\mathcal{R}$ also has Boolean connectives in $\{\wedge,\vee,\Rightarrow\}$, the explanatory relation $\rhd_{\mathcal{C}}$ satisfies additional  logical properties.

\medskip
\begin{lemma}
\label{cutting preservation}
If $\mathcal{C}_\varphi$ is a cutting for $T$ and $\varphi$, $Mod(\varphi') \subseteq Mod(\varphi)$ and $Mod(\varphi')\neq Triv$, then $\mathcal{C}_{\varphi'} = \{\mathbb{M} \cap Mod(\varphi') \mid \mathbb{M} \in \mathcal{C}_\varphi\}\setminus Triv $ is a cutting for $T$ and $\varphi'$.
\end{lemma}
\begin{proof}
For every $\mathbb{M} \in \mathcal{C}_\varphi$, we have that $\mathbb{M} \cap Mod(\varphi') \subseteq Mod(T \cup \{\varphi'\})$. 
It is clear that if $\mathcal{C}_{\varphi'}$ is not well-founded so it is $\mathcal{C}_\varphi$. Moreover, $\mathcal{C}_{\varphi'}$ is closed by union of sets. 
%Hence, for every maximal chain $C$ of $\mathcal{C}$, there is a unique maximal chain $C$ of $\mathcal{C}$ such that $C = \{\mathbb{M} \cap Mod(\varphi') \mid \mathbb{M} \in \mathcal{C}\}$. All other 
Thus all the conditions for defining a cutting are satisfied.% as well.
\end{proof}

\begin{lemma}
\label{cutting preservation two}
If $\mathcal{C}_\varphi$ is a cutting for $T$ and $\varphi$ and $Mod(\varphi') \cap Mod(T\cup \Set{\varphi})\neq Triv$, then $\mathcal{C}_{\varphi'} = \{\mathbb{M} \cap Mod(\varphi') \mid \mathbb{M} \in \mathcal{C}_\varphi\}\setminus Triv $ is a cutting for $T$ and $\varphi'$.
\end{lemma}
\begin{proof}
Similar to the one of the previous lemma.
\end{proof}

%\Isa{
%\begin{lemma}
%\label{other cutting prop}
%Let $\mathcal{R}$ be a satisfaction system, which has conjunction and disjunction. If $\mathcal{C}$ is a cutting for $T$ and $\varphi$, and $\mathcal{C}'$ is a cutting for $T$ and $\varphi'$, then $\mathcal{C} \cap \mathcal{C}'$ is a cutting for $T$ and $\varphi \wedge \varphi'$, and  $\mathcal{C} \cup \mathcal{C}'$ is a cutting for $T$ and $\varphi \vee \varphi'$.  
%\end{lemma}
%}

%\Isa{\fbox{Ca me semble \'evident mais j'ai pu rater un truc...}}

In the following results, we 
%write $\varphi \rhd_{\mathcal{C}_\varphi} \psi$ instead of $\varphi \rhd_{\mathcal{C}} \psi$ for an explanatory relation $\rhd_\mathcal{C}$ defined on cuttings such that the cutting for $\varphi$ is $\mathcal{C}_\varphi$. We 
also assume an additional structure on $\mathcal{C}$, according to Lemma~\ref{cutting preservation} and Lemma~\ref{cutting preservation two}, by imposing the following constraints:
\begin{equation}
Mod(\varphi') \subseteq Mod(\varphi), Mod(\varphi')\neq Triv \Longrightarrow \mathcal{C}_{\varphi'} = \{\mathbb{M} \cap Mod(\varphi') \mid \mathbb{M} \in \mathcal{C}_\varphi\}\setminus Triv
\label{eq:cuttingInclusion}
\end{equation}
\begin{equation}
Mod(\varphi') \cap Mod(T\cup \Set{\varphi})\neq Triv \Longrightarrow \mathcal{C}_{\varphi \wedge \varphi'} = \{\mathbb{M} \cap Mod(\varphi') \mid \mathbb{M} \in \mathcal{C}_\varphi\}
\label{eq:cuttingConjunction}
\end{equation}
\begin{equation}
\mathcal{C}_{\varphi \vee \varphi'} \mbox{ cutting for } \varphi \vee \varphi' \Longrightarrow \mathcal{C}_\varphi = \{\mathbb{M} \cap Mod(\varphi) \mid \mathbb{M} \in \mathcal{C}_{\varphi \vee \varphi'}\}
\label{eq:cuttingDisjunction}
\end{equation}
\begin{equation}
\mathcal{C}_{\varphi \Rightarrow \varphi'} \mbox{ cutting for } \varphi \Rightarrow \varphi' \Longrightarrow \mathcal{C}_\varphi = \{\mathbb{M} \cap Mod(\varphi') \mid \mathbb{M} \in \mathcal{C}_{\varphi \Rightarrow \varphi'}\}
\label{eq:cuttingImplication}
\end{equation}

\medskip
\begin{theorem}
\label{logical properties 2}
Let $\mathcal{R}$ be a satisfaction system with conjunction, disjunction and implication. Let $T$ be a KB, and $\mathcal{C}$ a set of cuttings satisfying Equation~\ref{eq:cuttingInclusion}--\ref{eq:cuttingImplication}. The following properties are satisfied, for every $\varphi,\varphi', \psi, \psi'$:
%Isa: j'ai decale les numeros pour que chaque propriete ait un numero different
\begin{enumerate}\setcounter{enumi}{4}
\item If $\varphi \rhd_{\mathcal{C}} \psi$ and $Mod(T \cup \{\psi \wedge \psi'\}) \neq Triv$, then $\varphi \rhd_{\mathcal{C}} \psi \wedge \psi'$.
\item If $\varphi \rhd_{\mathcal{C}} \psi$ and $T \cup \{\psi\} \models \varphi'$, then $\varphi \wedge \varphi' \rhd_{\mathcal{C}} \psi$.
%where $\mathcal{C}_{\varphi \wedge \varphi'} = \{\mathbb{M} \cap Mod(\varphi') \mid \mathbb{M} \in \mathcal{C}_\varphi\}$.
\item If $\varphi \rhd_{\mathcal{C}} \psi \wedge \psi'$ and $T \cup \{\psi\} \models \psi'$, then $\varphi \rhd_{\mathcal{C}} \psi$.
\item If $\varphi \wedge \varphi' \rhd_{\mathcal{C}} \psi$ (and $Mod(\varphi') \cap Mod(T\cup \Set{\varphi})\neq Triv$), then $\varphi \rhd_{\mathcal{C}} \psi$.
%for every cutting $\mathcal{C}_\varphi$ satisfying Equation~\ref{eq:cuttingConjunction}.
%such that $\mathcal{C}_{\varphi \wedge \varphi'} = \{\mathbb{M}' \cap Mod(\varphi') \mid \mathbb{M}' \in \mathcal{C}_\varphi\}$.
\item If $\varphi \rhd_{\mathcal{C}} \psi$, then $\varphi \vee \varphi' \rhd_{\mathcal{C}} \psi$.
%for every cutting $\mathcal{C}_{\varphi \vee \varphi'}$ such that $\mathcal{C}_\varphi = \{\mathbb{M}' \cap Mod(\varphi) \mid \mathbb{M}' \in \mathcal{C}_{\varphi \vee \varphi'}\}$.
\item If $\varphi \vee \varphi' \rhd_{\mathcal{C}} \psi$ and $T \cup \{\psi\} \models \varphi$, then $\varphi \rhd_{\mathcal{C}} \psi$.
%where $\mathcal{C}_\varphi$ satisfies Equation~\ref{eq:cuttingDisjunction}.
%where $\mathcal{C}_\varphi = \{\mathbb{M} \cap Mod(\varphi) \mid \mathbb{M} \in \mathcal{C}_{\varphi \vee \varphi'}\}$.
\item If $\varphi \rhd_{\mathcal{C}} \psi \vee \psi'$ and $Mod(T\cup\Set{\psi})\neq Triv$, then  $\varphi \rhd_{\mathcal{C}} \psi$.
\item For every cutting $\mathcal{C}_\varphi$ such that $\preceq_{\mathcal{C}_\varphi}$ is total, if $\varphi \rhd_{\mathcal{C}} \psi$ and $\varphi \rhd_{\mathcal{C}} \psi'$, then  $\varphi \rhd_{\mathcal{C}} \psi \vee \psi'$.
\item If $(\varphi \Rightarrow \varphi') \rhd_{\mathcal{C}} \psi$ and $T \cup \{\psi\} \models \varphi$, then $\varphi' \rhd_{\mathcal{C}} \psi$.
%where $\mathcal{C}_\varphi$ satisfies Equation~\ref{eq:cuttingImplication}.
%$\mathcal{C}_\varphi = \{\mathbb{M} \cap Mod(\varphi') \mid \mathbb{M} \in \mathcal{C}_{\varphi \Rightarrow \varphi'}\}$.
\item If $\varphi \rhd_{\mathcal{C}} \psi$, then $\psi \rhd_{\mathcal{C}} \psi$, for $\mathcal{C}_\psi = \{ \mathbb{M} \cap Mod(\psi) \mid \mathbb{M} \in \mathcal{C}_\varphi \}$.
\end{enumerate}
 %By Lemma~\ref{cutting preservation}, the conditions between the two cuttings $\mathcal{C}$ and $\mathcal{C}'$ in \Isa{Properties $(6)$,$(8)-(10)$ and $(13)$} are correct.
\end{theorem}

\begin{proof}
%We prove the properties involving the conjunction. The other properties involving the other Boolean connectives are proved similarly.
\begin{enumerate}\setcounter{enumi}{4}
\item By hypothesis, there exists $\mathbb{M} \in Min(\mathcal{C}_\varphi)$ such that $Mod(T \cup \{\psi\}) \subseteq \mathbb{M}$. Obviously, we have that $Mod(T \cup \{\psi \wedge \psi'\}) \subseteq Mod(T \cup \{\psi\})$, and then as $T \cup \{\psi \wedge \psi'\})$ is consistent, we can deduce that $\varphi \rhd_{\mathcal{C}} \psi \wedge \psi'$.
\item By hypothesis, there exists $\mathbb{M} \in Min(\mathcal{C}_\varphi)$ such that $Mod(T \cup \{\psi\}) \subseteq \mathbb{M}$. As we have further $Mod(T \cup \{\psi\}) \subseteq Mod(\varphi')$, we have that $Mod(T \cup \{\psi\}) \subseteq \mathbb{M} \cap Mod(\varphi')$. 
    By Lemma~\ref{cutting preservation two}, $\mathcal{C}_{\varphi \wedge \varphi'}$ is a cutting, and then we obviously have that  $Min(\mathcal{C}_{\varphi \wedge \varphi'}) = \{\mathbb{M} \cap Mod(\varphi') \mid \mathbb{M} \in Min(\mathcal{C}_\varphi)\}$. We can then conclude that $\varphi \wedge \varphi' \rhd_{\mathcal{C}}\psi$.
\item  By hypothesis, there exists $\mathbb{M} \in Min(\mathcal{C}_\varphi)$ such that $Mod(T \cup \{\psi \wedge \psi'\}) \subseteq \mathbb{M}$. As $T \cup \{\psi\} \models \psi'$, we have that $Mod(T \cup \{\psi \wedge \psi'\}) = Mod(T \cup \{\psi\})$, whence we can deduce that $\varphi \rhd_{\mathcal{C}} \psi$.
\item Let us notice that if $\mathcal{C}_\varphi$ is a cutting and $Mod(\varphi') \cap Mod(T\cup \Set{\varphi})\neq Triv$ then so is $\mathcal{C}_{\varphi \wedge \varphi'}$ by Lemma~\ref{cutting preservation two}, and $Min(\mathcal{C}_{\varphi \wedge \varphi'}) = \{\mathbb{M}' \cap Mod(\varphi') \mid \mathbb{M}' \in Min(\mathcal{C}_{\varphi})\}$. By hypothesis, there exists $\mathbb{M} \in Min(\mathcal{C}_{\varphi \wedge \varphi'})$ such that $Mod(T \cup \{\psi\}) \subseteq \mathbb{M}$. Hence, there exists $\mathbb{M}' \in \mathcal{C}_\varphi$ such that $\mathbb{M} = \mathbb{M}' \cap Mod(\varphi')$, and then $Mod(T \cup \{\psi\}) \subseteq \mathbb{M}'$, whence we can deduce that $\varphi \rhd_{\mathcal{C}} \psi$.
\item If $\mathcal{C}_{\varphi \vee \varphi'}$ is a cutting, then so is $\mathcal{C}_\varphi$ by Lemma~\ref{cutting preservation}. By hypothesis, there exists $\mathbb{M} \in Min(\mathcal{C}_\varphi)$ such that $Mod(T \cup \{\psi\}) \subseteq \mathbb{M}$. Hence, there exists $\mathbb{M}' \in Min(\mathcal{C}_{\varphi \vee \varphi'})$ such that $\mathbb{M} = \mathbb{M}' \cap Mod(\varphi)$, and then $Mod(T \cup \{\psi\}) \subseteq \mathbb{M}'$, whence we can conclude that $\varphi \vee \varphi' \rhd_{\mathcal{C}} \psi$.
\end{enumerate}
Properties (10), (13) and (14) can be proved similarly to Property (6), and Property (11) is a direct consequence of the fact that $Mod(\psi) \subseteq Mod(\psi \vee \psi')$. \\ Let us finish by the proof of Property (12). By hypothesis that $\preceq_{\mathcal{C}_\varphi}$ is total, the poset $(\mathcal{C}_\varphi,\subseteq)$ contains a unique maximal chain $C$. Hence, both $Mod(T \cup \{\psi\})$ and $Mod(T \cup \{\psi'\})$ are included in the unique minimal element $\mathbb{M}$ of $C$. Obviously, we have that $Mod(T \cup \{\psi \vee \psi'\}) \subseteq \mathbb{M}$, whence we can conclude that $\varphi \rhd_{\mathcal{C}} \psi \vee \psi'$.
%\Isa{The other properties can be proved in a similar way.\\ \fbox{Ou faut-il le faire explicitement ? au moins pour expliquer l'ordre total dans 12 ?}}
\end{proof}

Properties (6), (12) and (14) are extensions of the postulates \ECM, \ROR\ and \ERef\ defined in~\cite{PPU99}. Properties (8) and (9) are revisited forms of the postulates \ECC\ and \LOR, adapted to satisfaction systems and explanations based on cuttings.

%\Marc{The properties $(2)$ and $(8)$ are known in~\cite{PPU99} under the acronyms E-CM and ROR, and the properties $(4)$ and $(5)$ are revisited forms of postulates E-C-cut and LOR.}

The following two results are easy to prove, and therefore we omit the proof.

\medskip 
\begin{lemma}
Let $\mathcal{C}_\varphi$ and $\mathcal{C}_{\varphi'}$ be cuttings for $T$ and $\varphi$ (respectively $\varphi'$).
Then the set $\mathcal{C}_\varphi\bigoplus \mathcal{C}_{\varphi'}=\Set{A\cup B \mid A\in \mathcal{C}_\varphi \mbox{ and } B\in \mathcal{C}_{\varphi'} }$
is a cutting for $T$ and $\varphi\vee\varphi'$.
\end{lemma}

\medskip

\begin{corollary}
Suppose that $\rhd_\mathcal{C}$ is an explanatory relation defined on cuttings such that $\mathcal{C}_\varphi$ and $\mathcal{C}_{\varphi'}$ are the cuttings for $T$ and $\varphi$ and for $T$ and $\varphi'$ respectively. Furthermore, suppose that the cutting for $\varphi\vee\varphi'$ is precisely
$\mathcal{C}_\varphi\bigoplus \mathcal{C}_{\varphi'}$. If $\varphi\rhd_\mathcal{C}\psi$ or $\varphi'\rhd_\mathcal{C}\psi$, then
$\varphi\vee\varphi'\rhd_\mathcal{C}\psi$.
\end{corollary}

This result is a stronger version of \EDR.

Concerning \ERC\ we have to note that property (8) is a stronger version of \ERC.

%% file: retraction.tex
\section{Cutting based on retraction}
\label{sec:retraction}

In this section, we introduce some more constraints on cuttings, to be able to propose concrete examples in various logics in Section~\ref{applications}. The idea is to define particular cuttings from ``retractions'', that consist in transforming any KB $T$ into a new consistent one $T'$ such that $Mod(T') \subseteq Mod(T)$. This notion of retraction draws inspiration from Bloch \& al.'s works in~\cite{BL02,IB:arXiv-18,BPU04} on Morpho-Logics where some retractions have been defined based on erosions from mathematical morphology~\cite{BHR06} (see Section~\ref{applications}). Here, we propose to generalize this notion in the framework of satisfaction systems, and a retraction is defined on $Sen$ as follows.

\medskip
\begin{definition}[Retraction]
\label{retraction}
A {\bf retraction} is a mapping $\kappa : Sen \to Sen$ satisfying, for every $\varphi \in Sen$ such that $Mod(\varphi) \neq Triv$, the two following properties:
\begin{itemize}
\item {\bf Anti-extensivity:} $Mod(\kappa(\varphi)) \subseteq Mod(\varphi)$.
\item {\bf Vacuum:} $\exists k \in \mathbb{N}, Mod(\kappa^k(\varphi)) = Triv$ where $\kappa^0$ is the identity mapping, and for all $k > 0$, $\kappa^k(\varphi) = \kappa(\kappa^{k-1}(\varphi))$. %\Isa{est-ce qu'il ne faut pas enlever les cas o\`u $Mod(\varphi)$ couvrirait tous les mod\`eles possibles ? (en tout cas en PL il faudrait le faire je crois...)}
\end{itemize}
\end{definition}
The condition for $\varphi$ not to be a tautology in Definition~\ref{retraction} allows us to eliminate the trivial case where $(\bigwedge T) \wedge \varphi$ is a tautology for a KB $T$ and an observation $\varphi$, and then in this case $\varphi$ would not deserve any explanation from $T$.
\medskip
%\Isa{remplacer false par $\bot$ et true par $\top$ et dire avant qu'on les a dans la logique}

\begin{example}
Many examples of retractions can be defined in {\bf PL}. Here, we propose to define retractions from the tableau expansion rules given in Example~\ref{abduction via semantic tableau}. In Section~\ref{explanatory relations in PL}, we will study another retraction for {\bf PL} but based on erosions from mathematical morphology. \\ Let us recall that the tableau expansion rules break propositional formulas on their main Boolean connectives, and only the $\beta$-rules require a choice. Hence, given a choice such as for instance choosing at each time the left element of the formula (i.e. the formula $\varphi_1$ in $\beta$-rules), we start by inductively defining a mapping $h : Sen \to Sen$ as follows:
\begin{enumerate}
\item $h(p) = p$ for every $p \in \Sigma$;
\item $h(\neg \neg \varphi) = h(\varphi)$;
\item $h(\neg \neg \bot) = \top$;
\item $h(\varphi_1 \wedge \varphi_2) = h(\varphi_1) \wedge h(\varphi_2)$;
\item $h(\neg(\varphi_1 \Rightarrow \varphi_2)) = h(\varphi_1) \wedge h(\neg \varphi_2)$;
\item $h(\neg(\varphi_1 \vee \varphi_2)) = h(\neg \varphi_1) \wedge h(\neg \varphi_2)$;
\item $h(\varphi_1 \vee \varphi_2) = h(\varphi_1)$;
\item $h(\varphi_1 \Rightarrow \varphi_2) = h(\neg \varphi_1)$;
\item $h(\neg (\varphi_1 \wedge \varphi_2)) = h(\neg \varphi_1)$.
\end{enumerate} 
We could have just as easily defined the mapping $h$ as follows: the first six cases are identical, and %\Isa{il manque des h a droite dans ce qui suit...}
\begin{itemize}
\item $h(\varphi_1 \vee \varphi_2) = h(\varphi_2)$ or $h(\varphi_1 \vee \varphi_2) = h(\varphi_1) \wedge h(\varphi_2)$;
\item $h(\varphi_1 \Rightarrow \varphi_2) = h(\varphi_2)$ or $h(\varphi_1 \Rightarrow \varphi_2) = h(\varphi_1) \wedge h(\varphi_2)$;
\item $h(\neg (\varphi_1 \wedge \varphi_2)) = h(\neg \varphi_2)$ or $h(\neg (\varphi_1 \wedge \varphi_2)) = h(\neg \varphi_1) \wedge h(\neg \varphi_2)$.
\end{itemize}
Hence, each application of the mapping $h$ corresponds to a step of a path in the tableau. By structural induction on $\varphi$, it is easy to show that $Mod(h(\varphi)) \subseteq Mod(\varphi)$, hence $h$ verifies the anti-extensivity property. Moreover, it is also obvious to show that $h(h(\varphi)) = h(\varphi)$, and then except if $\varphi$ is an antilogy, we cannot have $Mod(h^k(\varphi)) = Triv$ for some $k$. This means that $h$ does not satisfy the vacuum property. Hence, $h$ is not a retraction. Now it is quite obvious to define a retraction from $h$. Indeed, given a mapping $h$ defined as previously, let us define the retraction $\kappa_h : Sen \to Sen$ as follows:  
$$\kappa_h(\varphi) = 
\left\{
\begin{array}{ll}
\bot & \mbox{if $h(\varphi) = \varphi$} \\
h(\varphi) & \mbox{otherwise}
\end{array}
\right.$$
It is not difficult to show that $\kappa_h$ is a retraction. 
\end{example}

%\Isa{montrer que h est anti-extensive et verifie vacuum ou a un point fixe}
Here, we introduce two cuttings based on retraction: $\mathcal{C}_{lcr}$ (last consistent retraction) and $\mathcal{C}_{lnr}$ (last non-trivial retraction). In $\mathcal{C}_{lcr}$ (respectively in $\mathcal{C}_{lnr}$), we define a unique sequence of ordered models (cf. Proposition~\ref{are total pre-orders}) which approximates the {\em most central part} of $T$ (respectively of $T \cup \{\varphi\}$).
%\Marc{Il faudrait peut-\^etre changer les indices "lc" et "lne" en "lcr" et "ltr" pour "last consistent retraction" et "last non-trivial retraction" ? Votre avis ?} \Isa{oui}
\medskip
\begin{definition}[Cuttings based on retraction]
\label{cutting based on retraction}
Let $\kappa$ be a retraction. Let $T \subseteq Sen$ be a KB and $\varphi$ be a sentence such that $T \cup \{\varphi\}$ is consistent. Let us define the two subsets of $\mathcal{P}(Mod(T \cup \{\varphi\})$ as follows: 
\begin{equation}
%\begin{center}
\mathcal{C}_{lcr} = \{Mod(\kappa^k(\bigwedge T) \wedge \varphi) \mid k \in \mathbb{N},Mod(\kappa^k(\bigwedge T) \wedge \varphi) \neq Triv\}
\label{eq:lcr}
\end{equation}
\begin{equation}
\mathcal{C}_{lnr} = \{Mod(\kappa^k(\bigwedge T \wedge \varphi)) \mid k \in \mathbb{N},Mod(\kappa^k(\bigwedge T \wedge \varphi)) \neq Triv\}
\label{eq:lnr}
\end{equation}
where $\bigwedge T = \varphi_1 \wedge \ldots \wedge \varphi_n$ if $T = \{\varphi_1 \wedge \ldots \wedge \varphi_n\}$. 
\end{definition}

\medskip
\begin{proposition}
$\mathcal{C}_{lcr}$ and $\mathcal{C}_{lnr}$ as defined in Equations~\ref{eq:lcr} and~\ref{eq:lnr} are cuttings for $T$ and $\varphi$.
\end{proposition}
\begin{proof}
Let us observe that in $\mathcal{C}_{lcr}$ (respectively in $\mathcal{C}_{lnr}$) we have a unique maximal chain of finite size the least element of which is $Mod(\kappa^n(\bigwedge T) \wedge \varphi)$ (respectively $Mod(\kappa^n(\bigwedge T \wedge \varphi))$ where $n = \sup \{k \in \mathbb{N} \mid Mod(\kappa^n(\bigwedge T) \wedge \varphi) \neq Triv\}$ (respectively  $Mod(\kappa^n(\bigwedge T \wedge \varphi)) \neq Triv$) (by the vacuum property such a $n$ exists). Obviously, both sets are closed under set-theoretical inclusion and are well-founded.
\end{proof}

\begin{proposition}
\label{are total pre-orders}
Both $\preceq_{\mathcal{C}_{lcr}}$ and $\preceq_{\mathcal{C}_{lnr}}$ derived from the cuttings defined in Equations~\ref{eq:lcr} and~\ref{eq:lnr} as in Equation~\ref{eq:relationFromCutting} are total pre-orders.
\end{proposition}

\begin{proof}
This is a direct consequence of the fact that both $\mathcal{C}_{lcr}$ and $\mathcal{C}_{lnr}$ have a unique maximal chain. 
\end{proof}

Following Definition~\ref{explanatory relation}, these two cuttings give rise to two explanatory relations defined as follows:
$$\varphi \rhd_{\mathcal{C}_{lcr}} \psi \Longleftrightarrow 
\left\{
\begin{array}{l}
%T {\not \models}_\Sigma \varphi, \{\psi\} {\not \models}_\Sigma \varphi, \\
Mod(T \cup \{\psi\}) \neq Triv,~\mbox{and} \\
Mod(T \cup \{\psi\}) \subseteq Mod(\kappa^n(T) \cup \{\varphi\})
\end{array}
\right.$$
$$\varphi \rhd_{\mathcal{C}_{lnr}} \psi \Longleftrightarrow 
\left\{
\begin{array}{l}
%T {\not \models}_\Sigma \varphi, \{\psi\} {\not \models}_\Sigma \varphi, \\
Mod(T \cup \{\psi\}) \neq Triv,~\mbox{and} \\
Mod(T \cup \{\psi\}) \subseteq Mod(\kappa^n(T \cup \{\varphi\}))
\end{array}
\right.$$
where $n = \sup\{k \in \mathbb{N} \mid Mod(\kappa^k(T) \cup \{\varphi\}) \neq Triv\}$ (respectively $n = \sup\{k \in \mathbb{N} \mid Mod(\kappa^k(T \cup \{\varphi\})) \neq Triv\}$.

\medskip
\begin{corollary}
Both explanation relations $\rhd_{\mathcal{C}_{lcr}}$ and $\rhd_{\mathcal{C}_{lnr}}$ satisfy all the logical properties of Theorems~\ref{logical properties 1} and~\ref{logical properties 2}.
\end{corollary}

\begin{proof}
Again this is derived from the fact that both $\mathcal{C}_{lcr}$ and $\mathcal{C}_{lnr}$ have a unique maximal chain. Only the property $(12)$ of Theorem~\ref{logical properties 2} requires for $\preceq_\mathcal{C}$ to be total. This has been proved for $\preceq_{\mathcal{C}_{lcr}}$ and $\preceq_{\mathcal{C}_{lnr}}$ in Proposition~\ref{are total pre-orders}. 
\end{proof}

%\Marc{Il manque peut-\^etre un discours sur les "cuttings" $\mathcal{C}_{lc}$ and $\mathcal{C}_{lne}$ et leur utilit\'e, et aussi la d\'efinition d'autres cuttings ?}

An example showing that these two definitions may provide different explanations of the same $\varphi$ is given in the case of {\bf PL} in Section~\ref{explanatory relations in PL}.

%% file: applications.tex
\section{Applications}
\label{applications}

In this section, we illustrate our general approach by defining abduction based on retractions for the logics {\bf PL}, {\bf HCL}, {\bf FOL}, {\bf MPL} and the {\bf DL} $\mathcal{ALC}$.
%We further develop 
%\Isa{In} the case of {\bf DL}
%s in Section~\ref{} by defining 
%several concrete retractions \Isa{are defined} for different fragments of the DL $\mathcal{ALC}$. 

\subsection{Explanatory relations based on retraction in {\bf PL}}
\label{explanatory relations in PL}

Here, drawing inspiration from Bloch \& al.'s works in~\cite{BL02,IB:arXiv-18,BPU04} on Morpho-Logics, we define retractions based on erosions from mathematical morphology~\cite{BHR06}. To define retractions in {\bf PL}, we will apply set-theoretic morphological operations. First, let us recall basic definitions of erosion in mathematical morphology~\cite{BHR06}. In complete lattices, an algebraic erosion is an operator that commutes with the infimum of the lattices. Concrete definitions of erosions often involve the notion of structuring element. Let us first consider the case of a lattice defined as the power set of some set (e.g. $\mathbb{R}^n$), with the inclusion relation. Let $X$ and $B$ be two subsets of $\mathbb{R}^n$. The erosion of $X$ by the structuring element $B$, denoted by $E_B(X)$, is defined as follows:
$$E_B(X) = \{x \in \mathbb{R}^n \mid B_x \subseteq X\}$$
where $B_x$ denotes the translation of $B$ at $x$. 
More generally, erosions in any space can be defined in a similar way by considering the structuring element as a binary relationship between elements of this space.

In {\bf PL}, knowing that we can identify any propositional formula $\varphi$ with its set of interpretations $Mod(\varphi)$, this leads to the following erosion of a formula $\varphi$:
$$Mod(E_B(\varphi)) = \{\nu \in  Mod \mid B_\nu \subseteq Mod(\varphi)\}$$
where $B_\nu$ contains all the models that satisfy some relationship with $\nu$. The relationship standardly used is based on a discrete distance $\delta$ between models, and the most commonly used is the Hamming distance $d_H$ where $d_H(\nu,\nu')$ for two propositional models over a same signature $\Sigma$ is the number of propositional symbols that are instantiated differently in $\nu$ and $\nu'$. In this case, we can rewrite the erosion of a formula as follows:
$$Mod(E_B(\varphi)) = \{\nu \in  Mod \mid \forall \nu' \in Mod, \delta(\nu,\nu') \leq 1 \Rightarrow \nu' \in Mod(\varphi)\}$$
This consists in using the distance ball of radius 1 as structuring element. To ensure the non-consistency condition to our retraction based on erosion, we need to add a condition on distances, the {\em differentiation property}

\medskip
\begin{definition}[Differentiation property]
Let $\delta$ be a discrete distance over a set $S$. Let us note $\Gamma_x$ for every $x \in S$, the set $\Gamma_x = \{y  \in S \mid \delta(x,y) \leq 1\}$. The distance $\delta$ has the {\bf differentiation property} if  for every $x,y \in S$, $\Gamma_x \neq \Gamma_y$. 
\end{definition}
The Hamming distance trivially satisfies the differentiation property.

\medskip
\begin{proposition}
\label{retraction for PL}
$E_B$ is a retraction for finite signatures $\Sigma$, and when it is based on a distance between models that satisfies the differentiation property.
\end{proposition}

\begin{proof}
It is anti-extensive since any erosion defined from a reflexive relationship is anti-extensive. Since $\delta$ is a distance, $\delta(\nu, \nu) = 0$ for any $\nu$ and $\nu \in B_\nu$, and thus for every $\varphi$ and for every model $\nu \in Mod(E_B(\varphi))$, we have that 
%$\delta(\nu,\nu) = 0$, and then 
$\nu \in Mod(\varphi)$. \\ 
Let $\varphi$ be a propositional formula such that $Mod(\varphi) \neq Mod(\Sigma)$.
%\Isa{(et donc il faut enlever ces cas de la def de vacuum...)}. 
As $\delta$ satisfies the differentiation property, there necessarily exists a model $\nu \in Mod(\varphi)$ and a model $\nu' \in Mod(\Sigma) \setminus Mod(\varphi)$ such that $\delta(\nu,\nu') \leq 1$ and $\nu' \not\in Mod(\varphi)$. Hence, each application of $E_B$ removes at least one model. As $\Sigma$ is a finite signature, $Mod(\varphi)$ is finite, and then there is $k \in \mathbb{N}$ such that $Mod(E_B^k(\varphi)) = \emptyset$.~\footnote{As the negation is considered in {\bf PL}, the set $Triv$ is empty, and then the consistency of a formula $\varphi$ can be defined by the fact that $Mod(\varphi) = \emptyset$.}   
\end{proof}

Let us first illustrate the instantiation of the two proposed definitions of explanation, when the retraction is an erosion using a ball of the Hamming distance as structuring element. Let us consider three propositional variables $a, b, c$, a KB $T = \{ a \vee b \vee c \}$, and the observation to be explained $\varphi = (a \wedge \neg b \wedge c) \vee (a \wedge b \wedge \neg c) \vee (a \wedge \neg b \wedge \neg c)$. Models can be graphically represented as the vertices of a cube, as shown in Figure~\ref{fig:consistent_erosion} (for instance the bottom left vertex is $\neg a \wedge \neg b \wedge \neg c$ while the top right one is $a \wedge b \wedge c$).

\begin{figure}[htbp]
\centerline{\hbox{
\includegraphics[height=3.5cm]{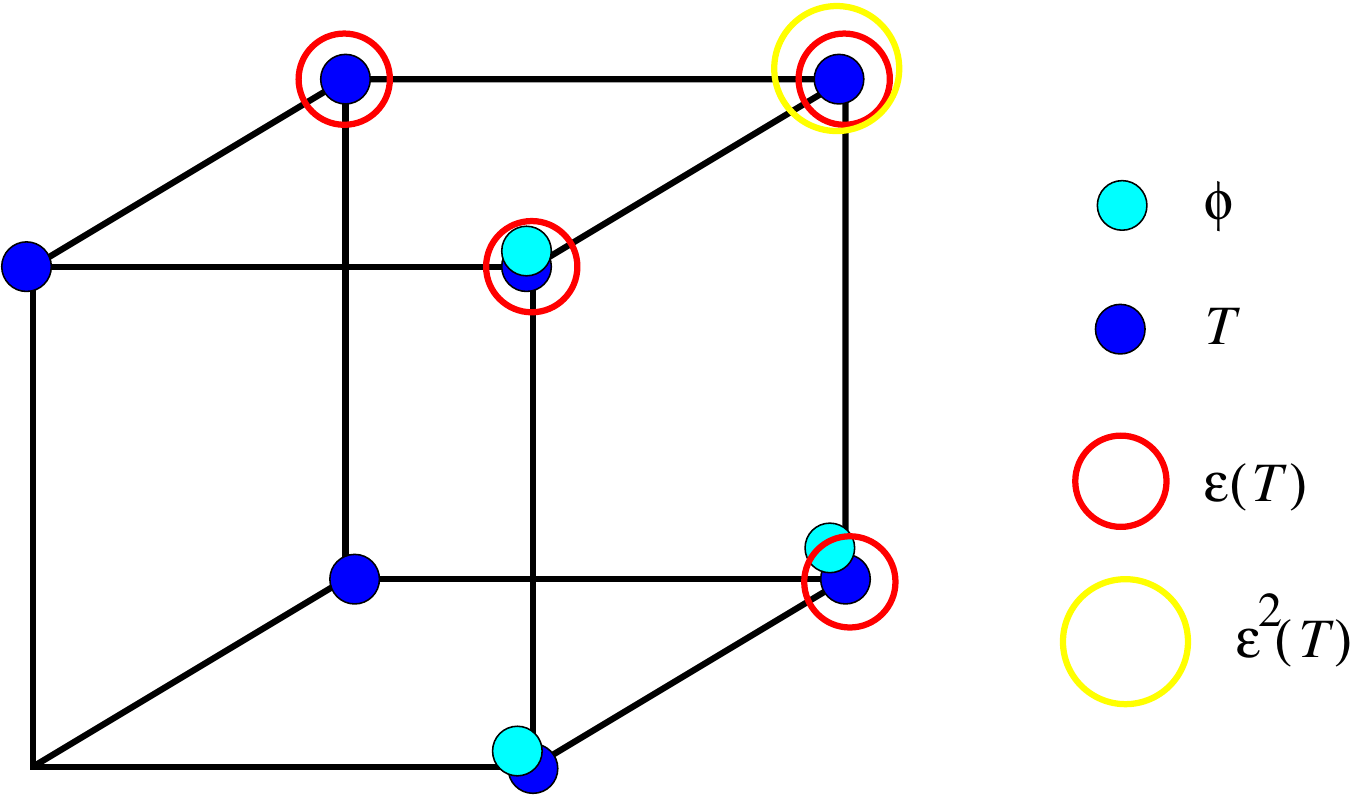}
}}
\caption{An example of last consistent erosion.}
\label{fig:consistent_erosion}
\end{figure}
It is easy to show that $\varepsilon(T)$ is consistent with $\varphi$, but $\varepsilon^2(T)$ is not~\cite{IB:arXiv-18}. Hence, for $\rhd_{\mathcal{C}_{lcr}}$ explanations $\psi$ are such that $Mod(\psi) \subseteq \{ (a \wedge \neg b \wedge c) \vee (a \wedge b \wedge \neg c)\}$. Similarly, it is easy to see that $\varepsilon(T \wedge \varphi) = \bot$, hence the explanations $\psi$ for $\rhd_{\mathcal{C}_{lnr}}$ are such that $Mod(\psi) \subseteq Mod(T \wedge \varphi)$. In particular $\psi = (a \wedge \neg b \wedge \neg c)$ is a potential explanation of $\varphi$ for $\rhd_{\mathcal{C}_{lnr}}$ but not for $\rhd_{\mathcal{C}_{lcr}}$. This is an example where the two proposed explanatory relations introduced in Definition~\ref{cutting based on retraction} provide different results.

Let us now show that the choice of the structuring element used in the erosions can impact the obtained explanations. This example is adapted from~\cite{IB:arXiv-18}. Let us consider the explanatory relation $\rhd_{\mathcal{C}_{lcr}}$, the KB $T = \{ a \Rightarrow c, b \Rightarrow c, a \vee b\}$, and the observation $c$. 
\begin{enumerate}
\item With the standard ball $B_\omega = \{ \omega' \in Mod \mid d_H(\omega, \omega') \leq 1 \}$, where $d_H$ denotes the Hamming distance, we get $\varepsilon^1(T) =\bot$. Thus, we have in particular,
\[
c \rhd_{\mathcal{C}_{lcr}} (a\vee b).
\]
\item Now we use $B_\omega^{ab} = \{ \omega' \in B_\omega \mid \omega(x) = \omega'(x) \mbox{ for all } x \notin \{ a, b \} \}$, 
i.e. $B_\omega^{ab}$ contains the valuations in $B_\omega$ that agree with $\omega$ outside $\{ a, b\}$. 
Then $\varepsilon^1 (T)=a\wedge b \wedge c$ and $\varepsilon^2 (T)=\bot$. Thus
\[
c \rhd_{\mathcal{C}_{lcr}} (a\wedge b).
\]
Notice that  $c\not\!\rhd_{\mathcal{C}_{lcr}} (a\vee b)$.
\item Finally, let us consider the
following structuring element
\[
B^{ab}_{\omega,2}=\{\omega\}\cup\{\omega'\in Mod \mid d_H(\omega,\omega')=2 \mbox{ and }
\omega(x)=\omega'(x)\;\mbox{for all $x\not\in \{ a, b \}$}\}
\]
Then $\varepsilon^1(T)=\varepsilon^2(T) = (\neg a\wedge b\wedge c) \vee (a\wedge \neg b\wedge c)$.
Thus,
\[
c \rhd_{\mathcal{C}_{lcr}} (a\wedge\neg b)\vee(\neg a\wedge  b).
\]
Notice that  $c\not\!\rhd_{\mathcal{C}_{lcr}}(a\wedge b)$.
\end{enumerate}
These different results can be interesting in situations where different explanations may be expected. This is illustrated by the following examples, that are further discussed in~\cite{IB:arXiv-18}:
\begin{enumerate}
\item
\[
\begin{array}{lcl}
a & = & \mbox{\em rained\_last\_night} \\
b & = & \mbox{\em sprinkle\_was\_on}\\
c & = & \mbox{\em grass\_is\_wet}
\end{array}
\]
The ``common sense cautious explanation'' of $c$ is $a\vee b$.
\item
\[
\begin{array}{lcl}
a & = & \mbox{\em low\_taxes} \\
b & = & \mbox{\em investment\_increases}\\
c & = & \mbox{\em economy\_grows}
\end{array}
\]
An explanation that enhances the chances of achieving the goal of making the economy to grow is $a\wedge b$.
\item
\[
\begin{array}{lcl}
a & = & \mbox{\em book\_was\_left\_somewhere else} \\
b & = & \mbox{\em somebody\_took\_the book}\\
c & = & \mbox{\em book\_is\_not\_in\_the shelf}
\end{array}
\]
An explanation  based on the principle of the ``Ockham's razor'' will select either $a$ or $b$ but not both, that is to say, $(a\wedge\neg b)\vee(\neg a\wedge b)$.
\end{enumerate}

\subsection{Explanatory relations based on retraction in {\bf HCL}}

Following our work in~\cite{AABH18} on the definition of revision operators based on relaxations for {\bf HCL}, we propose here to extend retractions that we have defined in the framework {\bf PL} to deal with the Horn fragment of propositional formulas. First, let us recall some useful notions.

\medskip
\begin{definition}[Model intersection]
Given a propositional signature $\Sigma$ and two $\Sigma$-models $\nu,\nu': \Sigma \to \{0,1\}$, we note $\nu \cap \nu' : \Sigma \to \{0,1\}$ the $\Sigma$-model defined by: 
$$p \mapsto 
\left\{ 
\begin{array}{ll}
1 & \mbox{if $\nu(p) = \nu'(p) = 1$} \\
0 & \mbox{otherwise}
\end{array}
\right.$$
Given a set of $\Sigma$-models $\mathcal{S}$, we note 
$$cl_\cap(\mathcal{S}) = \mathcal{S} \cup \{\nu \cap \nu' \mid \nu,\nu' \in \mathcal{S}\}$$
which is the closure of $\mathcal{S}$ under intersection of positive atoms. 
\end{definition}

For any set $\mathcal{S}$ closed under intersection of positive atoms, there exists a Horn sentence $\varphi$ that defines $\mathcal{S}$ (i.e. $Mod(\varphi) = \mathcal{S}$). Given a distance $\delta$ between models, we then define a retraction $\kappa$ as follows: for every Horn formula $\varphi$, $\kappa(\varphi)$ is any Horn formula $\varphi'$ such that $Mod(\varphi') = cl_\cap(Mod(E_B(\varphi))$ (by the previous property, we know that such a formula $\varphi'$ exists). 

\medskip
\begin{proposition}
With the same conditions as in Proposition~\ref{retraction for PL}, the mapping $\rho$ is a retraction. 
\end{proposition}

Again both explanatory relations $\rhd_{lcr}$ and $\rhd_{lnr}$ can be defined from $\kappa$ using Definition~\ref{cutting based on retraction}. 

\subsection{Explanatory relations based on retraction in {\bf FOL}}

A trivial way to define a retraction in {\bf FOL} is to map any formula to an antilogy. A less trivial and more interesting retraction consists in replacing existential quantifiers involved in the formula to be retracted by universal ones.  A dual approach has been adopted in~\cite{AABH18} for defining revision operators using dilations in {\bf FOL}. In the following we suppose that, given a signature, every formula $\varphi$ in $Sen$ is a conjunction of formulas in prenex form (i.e. $\varphi$ is of the form $\bigvee_{j}Q^j_1x^j_1 \ldots Q^j_{n_j} x^j_{n_j}. \psi_j$ where each $Q^j_i$ is in $\{\forall,\exists\}$). Let us define the retraction $\kappa$ as follows, for an antilogy $\tau$:
\begin{itemize}
\item $\kappa(\tau) = \tau$;
\item $\kappa(\forall x_1 \ldots \forall x_n. \varphi) = \tau$;
\item Let $\varphi = Q_1x_1 \ldots Q_n x_n. \psi$ be a formula such that the set $E_\varphi = \{i, 1 \leq i \leq n \mid Q_i = \exists\} \neq \emptyset$. Then, $\kappa(Q_1x_1 \ldots Q_n x_n. \varphi) = \bigvee_{i \in E_\varphi} \varphi_i$ where $\varphi_i = Q'_1x_1 \ldots Q'_n x_n.\psi$ such that for every $j \neq i$, $1 \leq j \leq n$, $Q'_j = Q_j$ and $Q'_i = \forall$;
\item $\kappa(\bigwedge_{j}Q^j_1x^j_1 \ldots Q^j_{n_j} x^j_{n_j}. \psi) = \bigwedge_{j}\kappa(Q^j_1x^j_1 \ldots Q^j_{n_j} x^j_{n_j}. \psi)$.
\end{itemize}

%\Isa{ne faut il pas remplacer $\vee$ par $\wedge$ dans la definition ?}

\medskip
\begin{proposition}
\label{retraction pour FOL}
$\kappa$ is a retraction. 
\end{proposition}

\begin{proof}
$\kappa$ is obviously anti-extensive, and satisfies the vacuum property because in a finite number of steps, we always reach the antilogy $\tau$. 
\end{proof}

\begin{example}
To illustrate our approach, let us consider the example taken from~\cite{MP93} and defined by the KB which only contains the formula $\forall x. \forall y. \forall z. (p(x,y) \wedge p(y,z) \Rightarrow p(x,z))$ and the observation $\varphi = \exists w. p(w,w)$. According to the explanatory relation we consider (i.e. either $\rhd_{lcr}$ or $\rhd_{lnr}$), the retracted formula will be different. For $\rhd_{lcr}$, only the formula $\forall x. \forall y. \forall z. (p(x,y) \wedge p(y,z) \Rightarrow p(x,z))$ is retracted. But in this case, to preserve consistency, the maximum number of retraction steps to apply is $0$. Hence, we have many possible explanations such as the trivial one $\exists w. p(w,w)$ (i.e. $\varphi \rhd_{lcr} \varphi$). The minimal explanation $\exists x. \exists y. p(x,y) \wedge p(y,x)$ given in~\cite{MP93} also satisfies $\varphi \rhd_{lcr} \exists x. \exists y. p(x,y) \wedge p(y,x)$.

%\Isa{J'ai l'impression que $T$ est parfois utilis\'e au lieu de $\varphi$...}

For $\rhd_{lnr}$, we can directly retract the formula $\forall x. \forall y. \forall z. (p(x,y) \wedge p(y,z) \Rightarrow p(x,z)) \wedge \exists w. p(w,w)$, but in this case to preserve consistency, the maximum number of retraction steps is $0$, and we come up with the previous case. Now, we can also consider the prenex form of $\forall x. \forall y. \forall z. (p(x,y) \wedge p(y,z) \Rightarrow p(x,z)) \wedge \exists w. p(w,w)$ which is $\forall x. \forall y. \forall z. \exists w. (p(x,y) \wedge p(y,z) \Rightarrow p(x,z)) \wedge p(w,w)$. Here, to preserve consistency, the maximum number of retraction steps to apply is $1$. We then obtain the formula $\forall x. \forall y. \forall z. \forall w. (p(x,y) \wedge p(y,z) \Rightarrow p(x,z)) \wedge p(w,w)$, and then a possible explanation here is $\varphi \rhd_{lnr} \forall w. p(w,w)$. In contrast, we now have that $\varphi {\not\!\rhd_{lnr}} \exists w. p(w,w)$ and $\varphi {\not\!\rhd_{lnr}} \exists x. \exists y. p(x,y) \wedge p(y,x)$.
\end{example}

\subsection{Explanatory relations based on retraction in {\bf MPL}}

By the classical first-order correspondence of {\bf MPL}, we can easily adapt the retraction defined for {\bf FOL} by replacing $\Diamond$ by $\Box$. Now, we can go further when dealing with formulas of the form $\Box \ldots \Box \varphi$. Indeed, in {\bf MPL}, we have that $Mod(\varphi) \subseteq Mod(\Box \varphi)$. Hence, we can remove in formulas the most external $\Box$. Of course, when dealing with modal logics such as $T$, $S4$, $B$ and $S5$ (i.e. the accessibility relation of Kripke models is always reflexive) where the formula $\Box \varphi \Rightarrow \varphi$ is a tautology, this is of no interest because in this case we have that $\varphi \equiv \Box \varphi$. This gives rise to the following retraction $\kappa$: here also we suppose that every formula $\varphi$ in $Sen$ is a conjunction of formulas $\varphi$ in the following normal form $\varphi = M_1 \ldots M_n. \psi$ where each $M_i$ is in $\{\Box,\Diamond\}$). 
\begin{itemize}
\item $\kappa(\tau) = \tau$ if $\tau$ is any antilogy;
\item $\kappa(\varphi) = \tau$ if $\varphi$ is modality free;
\item $\kappa(\Box \varphi) = \varphi$; %\Isa{ou le contraire plutot ?}
\item Let $\varphi = M_1 \ldots M_n \psi$ be a formula such that the set $E_\varphi = \{i, 1 \leq i \leq n \mid M_i = \Diamond\} \neq \emptyset$. Then, $\kappa(M_1 \ldots M_n \varphi) = \bigvee_{i \in E_\varphi} \varphi_i$ where $\varphi_i = M'_1 \ldots M'_n.\psi_i$ such that for every $j \neq i$, $1 \leq j \leq n$, $M'_j = M_j$ and $M'_i = \Box$;
\item $\kappa(\bigwedge_{j}M^j_1 \ldots M^j_{n_j} \psi) = \bigwedge_{j}\kappa(M^j_1 \ldots M^j_{n_j} \psi)$. %\Isa{comme plus haut est ce ce n'est pas $\wedge$ a la place de $\vee$ ?}
\end{itemize} 

\medskip
\begin{proposition}
$\kappa$ is a retraction. 
\end{proposition}

\begin{proof}
Similarly to Propostion~\ref{retraction pour FOL} and the fact that $Mod(\varphi) \subseteq Mod(\Box \varphi)$, $\kappa$ is anti-extensive and satisfies the vacuum property. 
\end{proof}

\subsection{Explanatory relations based on retraction in {\bf DL}}
\label{explanatory relations in DL}

Abduction in DL can take different forms: concept abduction, TBox abduction, ABox abduction and knowledge base abduction (see e.g.~\cite{elsenbroich2006,HalBriKlaDL14,klarman2011}).

\medskip

\begin{definition}[Abduction types in DL]
Let $\mathcal{L}, \mathcal{L}'$ be two arbitrary description logics, $\K=(\T,\A)$ a knowledge base in $\mathcal{L}$ with $\T$ the TBox and $\A$ the Abox, $C, D$ two concepts in $\mathcal{L}$ satisfiable with respect to $\K$ (i.e. that admit non empty interpretations). Abduction forms in DL are as follows:\\
-- {\em Concept abduction:} given an observation concept $O$ in $\mathcal{L}$ satisfiable w.r.t. $\K$, the set of explanations introduced in Definition~\ref{set of explanations} writes as:  
\[
Expla_\K(O) = \{H \mid H \mbox{ satisfiable w.r.t. } \K 
%\not\models_\Sigma H {\sqsubseteq} O\\ 
~\mbox{and}~\K  \models H \sqsubseteq O\} \,.
\] 
The set of concepts in $Expla_\K(O)$ may possibly be expressed in  another description logic $\mathcal{L}'$.\\
-- {\em TBox abduction:} let $C \sqsubseteq D$ be satisfiable w.r.t. $\K$, the set of explanations is made of axioms  defined as:
\[
Expla_\K(C \sqsubseteq D) = \{E\sqsubseteq F  \mid E \sqsubseteq F  \mbox{ satisfiable w.r.t. } \K 
%&\{E \sqsubseteq F \} \not\models_\Sigma  C \sqsubseteq D\\
 ~\mbox{and}~\K \cup \{E \sqsubseteq F\}  \models C \sqsubseteq D\} \,.
\] 
%$S_T = \{E_i \sqsubseteq F_i \mid i\leq n\}$, possibly in $\mathcal{L}'$, such that $\K \cup S_T  \models C \sqsubseteq D$.
-- {\em ABox abduction:}  let $S_a$ be a set of assertions representing the observation, the set of explanations is the set  $S_b$ of
ABox assertions such that $S_b \mbox{ satisfiable w.r.t. } \K ~\mbox{and}~\K \cup S_b \models S_a$.\\
%S_b \not\models_\Sigma S_a$.
-- {\em KB abduction:} let $\{\varphi\}$ be a consistent set of ABox or TBox assertions w.r.t. $\K$. A solution
of knowledge base abduction, considered as a combination of TBox abduction and ABox abduction, is
any  finite set $S = \{\psi_i, i=1... n \}$ in $\mathcal{L}'$  satisfiable w.r.t. $\K$ and such that $\K \cup S \models \{\varphi\}$.
% S \not\models_\Sigma \{\varphi\}$.
\end{definition}
%\Marc{The TBox abduction is directly transposed in {\bf DL}. On the contrary, the concept abduction is not, but it is in {\bf CL}. Indeed, we have the following equivalences:
%$$\begin{array}{ll}
%\K \models^{DL}_\Sigma C \sqsubseteq D & \Longleftrightarrow \overline{\K} \models^{CL}_\Sigma C \Rightarrow D \\ 
%& \Longleftrightarrow \overline{\K} \cup \{C\} \models^{CL}_\Sigma D
%\end{array}$$
%where $\overline{K}$ is obtained from $\K$ by transforming every sentence $A \sqsubseteq B$ into $A \Rightarrow B$. \\ The set of explanations now writes as:
%\[
%\begin{split} Expla_{\overline{\K}}(O) = \{H \mid Mod(\overline{\K} \cup \{H\}) \neq \emptyset, H %\not\models_\Sigma O\\ ~\mbox{and}~\overline{\K} \cup \{H\}  \models_\Sigma O\} \,.
%\end{split}
%\]
%}

%\fbox{Check the existence of $\sqsubseteq-maximal$ in DL.}\\\fbox{Maybe need for finite model property.} 

As in any other logic, additional constraints can be used to find the preferred explanations in $Expla_\K$ (minimality, etc. (see e.g.~\cite{bienvenu2008}).

Let us illustrate these notions on an example inspired from an image interpretation task.
\begin{figure}[htp]
  \centering
  \subfloat[Complete picture]{\label{fig:edge-a}\includegraphics[scale=0.55]{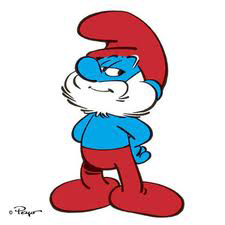}}
  \hspace{5pt}
  \subfloat[Region $a$]{\label{fig:contour-b}\includegraphics[scale=0.55]{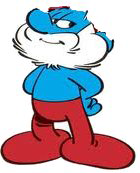}}
  \hspace{5pt}
  \subfloat[Region $b$]{\label{fig:contour-c}\includegraphics[scale=0.55]{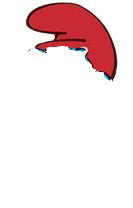}}
   \hspace{5pt}
  \subfloat[Region $c$]{\label{fig:contour-d}\includegraphics[scale=0.55]{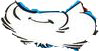}}
  \caption{Picture of Smurf and three regions $a$, $b$, $c$.}
  \label{fig:smurf}
\end{figure}
\begin{example}
\label{ex:smurf}
\footnote{This example results from a discussion with Felix Distel during his visit at LTCI, T\'el\'ecom ParisTech (summer 2013).} Suppose we have an image, in which we have identified three regions $a, b$ and $c$ (Figure~\ref{fig:smurf}). Region $b$ has been identified as \verb+Hat+ and has a color attribute
\verb+Red+, while region $c$ has been identified as \verb+Beard+. There is a spatial relation \verb+hasOnTop+ that
links $a$ to $b$, and a spatial relation \verb+hasPart+ linking $a$ to $c$. Furthermore, the background
knowledge tells us that smurf leaders are smurfs that wear red hats (i.e. have them on
top) and have beards.
In this example, a good approach  should be able to come up with the explanation that region a might be a smurf leader.

The background knowledge is encoded as a TBox:
\[
\begin{aligned}
\T= \{\verb+SmurfLeader+ &\sqsubseteq \exists \verb+hasPart.Beard+ \sqcap  \exists \verb+hasOnTop.RedHat+,\\
\verb+RedHat+ &\equiv \verb+Hat+  \sqcap \exists \verb+hasColor.Red+\}
\end{aligned}
\]
The observation is encoded as an ABox
\[
\begin{aligned}
\A_o= \{(a,b):& \verb+hasOnTop+,\\
            (a,c):&\verb+hasPart+,\\
            b:&\verb+Hat+,\\
            b:&\exists\verb+hasColor.Red+,\\
            c:&\verb+Beard+\}
\end{aligned}
\]
%In ABox abduction we look for a set $\A_e$ of assertions. % such that $\T \cup \A_e \models_\Sigma \A_o, \A_e \not\models_\Sigma \A_o$ and $\T \models_\Sigma \A_e$. 
%Obviously, there can be many such sets, which is why one uses certain minimality criteria. 
A preferred explanation for this example is $\A_e=\{a:\verb+SmurfLeader+\}$.

In Concept abduction, we need to say which region we want to explain, here region $a$. Then the \emph{most specific concept} for this region is: %If we use $\EL{}$ we get
\[
O=msc(a)= \exists hasPart.Beard \sqcap \exists hasOnTop.(Hat \sqcap hasColor.Red) .
\]
Concept Abduction now looks for a concept description $C$ such that
$\mathcal{T} \models C \sqsubseteq O$. A solution that is both length minimal\footnote{The length of a concept is defined as the number of atomic concepts appearing in it.} and $\sqsubseteq$-maximal would
be $C = SmurfLeader$ which is one of the expected solution.
\end{example}

In this paper, we consider the general form of abduction in DL, i.e. KB abduction. The other forms can be seen as particular cases. Explanatory relations  of a KB in DL can be defined in two ways: 
\begin{itemize} 
\item When the logic is equipped with the disjunction and full negation constructors, as it is the case of the logic $\ALC{}$ and its extensions,  the theory is transformed into an \emph{internalized concept} on which the retraction operators act. The internalized concept is defined as follows:   $C_{\T}:=\bigsqcap_{(C\sqsubseteq D) \in \T} (\neg C \sqcup D)$.  When Abox assertions are considered one can also internalize the Abox to an equivalent concept provided that nominals are part of the syntax. Nominals are concept descriptions having as semantics: $(\{o\})^\I=\{o^\I\}$, where $(\_^\I,\Delta^\I)$ is an interpretation. Then the Abox assertions are transformed into concept inclusions as follows: $a:C$ corresponds to $\{a\} \sqsubseteq C$, $(a,b):r$ corresponds to  $\{a\} \sqsubseteq r.\{b\}$, etc. 
%$C_{\A}:=\bigsqcap_{a \in \A} \exists U.\left(o_a\sqcap \bigsqcap_{a:C\in A}C\sqcap \bigsqcap_{(a,b):R \in \A}\exists R.o_b\right)$, where $U$ is a  role not occurring in $\A$ nor in $\T$ and $o_a$ and $o_b$ are the  nominals associated with the individuals $a$ and $b$ respectively. A knowledge base $K=(\T,\A)$ can hence be internalized as follows: $C_{\K}:=C_{\T} \sqcap C_{\A}$. %In what follows, a relaxation of a knowledge base, noted $\rho(\K)$, should be understood as a relaxation of its internalisation concept ($\rho(C_{\K})$).
\item When the logic does not allow for full negation, a possible workaround consists in retracting all the formulas at the same time. Hence, one needs to define concrete retraction operators both on concepts and on formulas. $\Sigma$-formula retraction can be defined in two ways (other definitions may also exist). For sentences of the form $C\sqsubseteq D$, a first possible approach consists in retracting the set of models of $D$ while the second one amounts to  ``relax'' the set of models of $C$ (see e.g.~\cite{AABH18} for definitions of relaxations in satisfaction systems, with several examples in DL).
\end{itemize}
  
Note also that retracting a concept (or formula) amounts to ``relaxing" its negated form, and can then be seen as its dual operator.  The notion of concept relaxation has been first introduced in description logics to define dissimilarity measures between concepts in~\cite{DAB14a,DAB14b}, and has been extended to define revision operators in arbitrary logics in~\cite{AABH18}. It is extensive and exhaustive, i.e. \ $\exists k \in \nat$ such that $\rho^k(C) \equiv \top$, where $\rho^k$ denotes $k$ iterations of a relaxation $\rho$. Relaxation of formulas have been defined from retraction of concepts to come up with revision operators, in particular within the context of description logics. In~\cite{AABH18}, some retraction operators of DL-concept descriptions have been introduced. These operators, designed for the purpose of revision are too strong, since their aim is to remove formulas in the background knowledge that are inconsistent with the new acquired one. The philosophy behind abduction is slightly different. One should add new knowledge to the set of consequences of the background theory. Hence, the retraction operator should act on the entire KB rather than just seeking a subpart that is more appropriate to revise.Concretely, this means that instead of relaxing the formulas, potentially each one to a different extent, as done for revision, for abduction, all formulas have to be retracted in the same way.

Hence, we need to introduce new retraction operators suited to the purpose of abduction. In what follows, we restrict ourselves to the context of the logic $\ALC{}$, as defined in Section~\ref{institutions}, possibly enriched with nominals.

\medskip

%\subsection{Abstract formulas retraction operators in DL}
\begin{definition}[Concept Retraction]
\label{concepterosion}
Let $\C{\Sigma}$ be the set of concept descriptions defined over a signature $\Sigma$. A \emph{(concept) retraction} is an operator $\kappa \colon \C{\Sigma}
  \rightarrow \C{\Sigma}$ that satisfies the following two properties for all
  $C \in \C{\Sigma}$ such that $C$ is not equivalent to $\top$:
  \begin{enumerate}
  %  \item $\kappa$ is \emph{non-decreasing}, i.e. $C \sqsubseteq D$ implies  $\kappa(C) \sqsubseteq \kappa(D)$,
    \item $\kappa$ is \emph{anti-extensive}, i.e.\ $\kappa(C) \sqsubseteq C $, and 
    \item $\kappa$ satisfies the {\em vacuum} property, i.e. $\exists k \in \mathbb{N}, \kappa^k(C) = \bot$ where $\kappa^0$ is the identity mapping, and for all $k > 0$, $\kappa^k(C) = \kappa(\kappa^{k-1}(C))$.
  %  \item $\kappa$ is \emph{exhaustive (vacuum property)}, i.e. $\forall D\in \C{\Sigma}, \exists k \in \mathbb{N} \text{ such that } \kappa^k(C) \sqsubseteq D$.
  \end{enumerate}
%  where $\kappa^k$ denotes $\kappa$ applied $k$ times, and $\kappa^0$ is the identity mapping.
\end{definition} 
This definition is a direct instantiation to DL of Definition~\ref{retraction}. Now we propose, as an example, the following operator to define a particular retraction in $\ALC$.

\medskip

\begin{definition}
Given an $\ALC$-concept description $C$ we define an operator $\kappa_f$
recursively as follows.
%For $C = A \in N_C$:
\begin{itemize}
\item For $C = A \in N_C$ (i.e. an atomic concept), $\kappa_f(C) = \bot$.
\item For $C = \neg A $, $\kappa_f(C) = \bot$.
\item For $C=\bot$,   $\kappa_f(C) = \bot$.
\item For $C = \top$  $\kappa_f(C) = \top$.
\item For $C = C_1 \sqcup C_2$,  
$\kappa_f(C_1 \sqcup C_2) = \left(\kappa_f(C_1) \sqcup C_2\right) \sqcap \left(C_1 \sqcup \kappa_f(C_2)\right)$.
\item For $C = C_1 \sqcap C_2$,   $\kappa_f(C_1 \sqcap C_2) = \kappa_f(C_1)\sqcap \kappa_f(C_2)$.
\item For $C=\forall r.D$, with $r\in N_R$,  $\kappa_f(C)=\forall r.\kappa_f(D)$.
\item  For $C=\exists r.D$,  $\kappa_f(C)=(\forall r.D) \sqcup \left(\exists r.\kappa_f(D)\right)$.
\end{itemize}
\label{def:retract_ALC}
\end{definition}
Note that this definition assumes that any concept is rewritten using standard De Morgan rules so that negations apply only on atomic concepts.
\medskip

\begin{proposition}
The operator $ \kappa_f$ is a retraction.
\end{proposition}

\begin{proof}
The proof is straightforward, by induction on the structure of $C$.
\end{proof}

%\Isa{This retraction operator can then be used as a basic component of the retractions defined in the general framework of satisfaction systems in Section~\ref{sec:retraction}, the specification to the case of DL being straightforward.}

\begin{example}
\label{retraction and inconsistencies}
Let us illustrate this retraction operator on Example~\ref{ex:smurf} using the explanatory relations introduced in Section~\ref{sec:retraction}. To ease the reading, we will note the concepts by capital letters and roles by small letters, e.g. B=\verb+Beard+, S=\verb+SmurfLeader+, H=\verb+Hat+, R=\verb+Red+, t=\verb+hasOnTop+, p=\verb+hasPart+,c=\verb+hasColor+.
The unfolded TBox writes as:
\[
\T=\{S\sqsubseteq\exists p.B \sqcap \exists t.(H\sqcap \exists c.R)\}
\]
and the observation writes as:
\[
\varphi = \exists p.B \sqcap \exists t.(H\sqcap \exists c.R)
\]
Note that $\T = \{ S \sqsubseteq \varphi\}$, and $C_\T = \neg S \sqcup \varphi$, where $C_\T$ is the internalized concept of the TBox.

First, note that we have:
\begin{equation}
\kappa_f(\varphi) = \kappa_f(\exists p.B \sqcap \exists t.(H \sqcap \exists c.R)) = \forall p.B \sqcap \forall t.(H \sqcap \exists c.R)
\label{eq:kappa1}
\end{equation}
and 
\begin{equation}
\kappa^2_f(\varphi) =  \kappa_f(\kappa_f(\varphi)) = \kappa_f(\forall p.B \sqcap \forall t.(H \sqcap \exists c.R)) = \bot 
\label{eq:kappa2}
\end{equation}

Let us now consider the two explanatory relations $\rhd_{\mathcal{C}_{lcr}}$ and $\rhd_{\mathcal{C}_{lnr}}$ defined in Section~\ref{sec:retraction}, applied here in $\ALC$ and with $\kappa_f$.

\begin{cas}[$\rhd_{\mathcal{C}_{lcr}}$]
The $\rhd_{\mathcal{C}_{lcr}}$ relation amounts to take the last retraction of $C_\T$ that is still consistent with $\varphi$, i.e.:
\[
%\begin{split}
\varphi \rhd_{\mathcal{C}_{lcr}} \psi \Leftrightarrow \psi \sqsubseteq \kappa^n_f(\neg S \sqcup \varphi)  \sqcap \varphi
%\sqcup \neg(\neg S \sqcup \varphi \sqcup \varphi)\\
%\models & \kappa^n_f(\neg S \sqcup \varphi) \sqcup (S \sqcap \neg \varphi)
%\end{split}
\]
We have:
\[
\begin{split}
\kappa^1_f(\neg S \sqcup \varphi)& = \varphi \sqcap (\neg S \sqcup \kappa^1_f(\varphi))\\
\kappa^2_f(\neg S \sqcup \varphi)&=\kappa^1_f(\varphi) \sqcap (\kappa^1_f(\varphi) \sqcap (\kappa^2_f(\varphi) \sqcup \neg S)) \\
& = \kappa^1_f(\varphi) \sqcap (\kappa^2_f(\varphi) \sqcup \neg S))\\
&=\kappa^1_f(\varphi) \sqcap \neg S, \text{ since }\kappa^2_f(\varphi) = \bot \\
\kappa^3_f(\neg S \sqcup \varphi)&=\bot
\end{split}
\]
Then $n=2$ and
\[
%\begin{split}
\psi \sqsubseteq (\kappa_f^1(\varphi) \sqcap \neg S) \sqcap \varphi
%\sqcup (S \sqcap \neg \varphi)\\
%\models & (\kappa_f^1(\varphi) \sqcap S) \sqcup 
% (\neg S \sqcap \neg \varphi)
%\end{split}
\]
with $\kappa_f^1(\varphi)=\forall p.B \sqcap \forall t.(H \sqcap \exists c.R)$. Since $\kappa_f$ is anti-extensive, $\kappa_f^1(\varphi) \sqcap \varphi = \kappa_f^1(\varphi)$, and
\[
\psi \sqsubseteq (\forall p.B \sqcap \forall t.(H \sqcap \exists c.R) \sqcap \neg S)
\]
\end{cas}
\begin{cas}[$\rhd_{\mathcal{C}_{lnr}}$]
The $\rhd_{\mathcal{C}_{lnr}}$ relation amounts to take the last non empty retraction of $C_\T \sqcap \varphi$, i.e.:
\[
%\begin{split}
\varphi \rhd_{\mathcal{C}_{lnr}} \psi \Leftrightarrow \psi \sqsubseteq
 \kappa^n_f((\neg S \sqcup \varphi)\sqcap \varphi) 
 %\sqcup \neg(\neg S \sqcup \varphi \sqcup \varphi)\\
%\models & \kappa^n_f(\varphi) \sqcup \neg(\neg S \sqcup \varphi)
%\end{split}
\]
with the largest possible value of $n$ such that the retraction is not empty.
Since we have $(\neg S \sqcup \varphi)\sqcap \varphi = \varphi$, according to Equations~\ref{eq:kappa1} and~\ref{eq:kappa2}, $n=1$, and 
%\[
%\begin{split}
%\kappa^1_f(\varphi)& = \kappa^1_f(\exists p.B \sqcap \exists t.(H\sqcap \exists c.R))\\
%&=\forall p.B \sqcap \forall t.(H \sqcap \exists c.R)\\
%\kappa^2_f(\varphi)& = \bot
%\end{split}
%\]
then
\[
\psi \sqsubseteq \forall p.B \sqcap \forall t.(H \sqcap \exists c.R) \sqcup (S \sqcap \neg  \varphi)
\]
A possible solution is then $S$ according to subset minimality, which well fits the intuition.
\end{cas}
\end{example}

Several other retractions could be proposed. In particular, several relaxations proposed in~\cite{AABH18} for revision could be modified to become retractions. For instance, relaxing $C \sqsubseteq D$ can be performed either by retracting $C$ or by relaxing $D$. Similarly, retracting $C \sqsubseteq D$ could be performed either by relaxing $C$ or by retracting $D$. 

%% file: conclusion.tex
\section{Conclusion}

In this paper, we proposed a new framework for abduction in satisfaction systems, by introducing the notion of cutting, which provides a structure on the set of models among which an explanation can be found. Inspired by previous work in propositional logic where abduction was defined from morphological erosions, we proposed to define cuttings from the more general notion of retraction, and prove a set of rationality postulates for the derived explanatory relations. The generic feature of the proposed approach is illustrated by providing concrete examples of retractions, cuttings and explanatory relations in various logics.

Future work will aim at further analyzing the structure of the set of cuttings $\mathcal{C}$ for a theory $T$, and the properties of the derived relation $\preceq_{\mathcal{C}}$. The examples in DL could also be further investigated, by considering other types of retractions as well as various fragments of $\ALC$, as done for revision in~\cite{AABH18}. Links between the proposed approach with other abduction methods could also deserve to be investigated, such as with sequent calculus, prime implicants~\cite{quine1955} or equational logic~\cite{echenim2012,echenim2013}. To address the question of uncertainty in the observations, or in the theory, the proposed approach could be extended to the case of fuzzy logic, based on our previous work on mathematical morphology in the framework of institutions~\cite{MA:arXiv-17}. Finally applications will be further developed, in particular for image understanding and spatial reasoning.